\documentclass[letterpaper]{article}
\usepackage{algorithm}
\usepackage{algorithmic}

\usepackage[numbib]{tocbibind}
\usepackage[square]{natbib}
\usepackage{natbib}
\setcitestyle{numbers,square}

\usepackage{amsmath}
\usepackage{amsthm}
\usepackage{booktabs}
\usepackage{thm-restate}
\usepackage{etoolbox}
\usepackage[]{complexity}
\usepackage{tikz,pgf}

\usepackage{pifont}
\usepackage{enumitem}
\usepackage{multirow}

\usepackage{newfloat}
\usepackage{listings}
\floatstyle{ruled}
\newfloat{listing}{tb}{lst}{}
\floatname{listing}{Listing}
\usepackage{comment}

\usepackage{cancel}
\usepackage{amssymb}
\usepackage{xspace}
\usepackage{color}
\usepackage{todonotes}

\usepackage{pifont}
\usepackage{enumitem}
\usepackage{multirow}
\newcommand{\dt}{\mathtt{t}}

\usepackage{thm-restate}
\usepackage{etoolbox}
\usepackage[]{complexity}
\usepackage{tikz,pgf}
\usetikzlibrary{shapes,backgrounds,shapes.geometric,shadows,patterns,positioning}
\tikzstyle{arg}=[draw,circle,fill=gray!15,inner sep=1pt,minimum size=.5cm]
\tikzstyle{targ}=[inner sep=1pt,minimum size=.5cm]
\tikzstyle{newarg}=[draw,circle,inner sep=1pt,minimum size=.5cm]
\tikzstyle{dt}=[inner sep=1pt,minimum size=0.2cm,font=\small,outer sep=-.5pt]

\tikzstyle{attack}=[->,left,thick,>=stealth]
\tikzstyle{abaarg}=[rectangle, draw, minimum width=15pt, minimum height=15pt]

\tikzstyle{IArg2}=[draw,ellipse,inner sep=2pt,minimum size=.5cm]
\tikzstyle{IArg1}=[ellipse split,draw,inner sep =2]

\usetikzlibrary{shapes.multipart}

\newcommand{\asms}{asms}
\newcommand{\theory}{\mathit{Th}}

\newtheorem{theorem}{Theorem}[section]
\newtheorem{lemma}[theorem]{Lemma}
\newtheorem{proposition}[theorem]{Proposition}
\newtheorem{corollary}[theorem]{Corollary}

\newtheorem{definition}[theorem]{Definition}
\newtheorem{example}[theorem]{Example}
\newtheorem{remark}[theorem]{Remark}

\newtheorem{construction}{Construction}

\newcommand{\stable}{{\mathit{stb}}}
\newcommand{\stb}{\stable}
\newcommand{\adm}{\mathit{ad}}
\newcommand{\prf}{\mathit{pr}}
\newcommand{\pref}{\mathit{pr}}
\newcommand{\grd}{\mathit{gr}}
\newcommand{\com}{\mathit{co}}
\newcommand{\comp}{\mathit{co}}

\newcommand{\cf}{\mathit{cf}}

\newcommand{\contrary}[1]{\overline{#1}}
\newcommand{\contraryempty}{\contrary{\phantom{a}}}
\newcommand{\cl}{\mathit{cl}}

\newcommand{\args}{\ensuremath{A}}

\newcommand{\attacker}{\ensuremath{\mathrm{Att}}}

\newcommand{\supporter}{\ensuremath{\mathrm{Sup}}}

\newcommand{\conc}{\mathit{conc}}
\newcommand{\CF}{\mathcal{F}}
\newcommand{\CG}{\mathcal{G}}
\newcommand{\HH}{\mathcal{H}}

\newcommand{\SigmaP}[1]{\ComplexityFont{\Sigma}_{#1}^{\P}}
\newcommand{\PiP}[1]{\ComplexityFont{\Pi}_{#1}^{\P}}

\newcommand{\Cred}{\textit{Cred}}

\newcommand{\Skept}{\textit{Skept}}

\newcommand{\Ver}{\textit{Ver}}

\newcommand{\F}{\mathit{F}}
\newcommand{\ie}{i.e., }

\newcommand{\FT}[1]{#1}
\title{Non-flat ABA is an Instance of Bipolar Argumentation}

\author{
	Markus Ulbricht$^1$, 
	Nico Potyka$^2$, 
	Anna Rapberger$^3$, and  
	Francesca Toni$^{3}$\\[1ex]
	$^1$Leipzig University, Germany\\
	$^2$Cardiff University, School of Computer Science and Informatics\\
	$^3$Imperial College London, UK\\[1ex]
	$^1$mulbricht@informatik.uni-leipzig.de \\
	$^2$PotykaN@cardiff.ac.uk,\\
	$^3$\{f.toni,a.rapberger\}@imperial.ac.uk
}

\begin{document}

	\maketitle
	
	\begin{abstract}
		Assumption-based Argumentation (ABA) is 
		a well-known structured argumentation formalism, whereby arguments and attacks between them are drawn from rules, defeasible assumptions and their contraries. 
		A common restriction imposed on ABA frameworks (ABAFs) is that they are \emph{flat}, \ie each of the defeasible assumptions can only be assumed, but not derived. 
		While it is known that flat ABAFs can be translated into abstract argumentation frameworks (AFs) as proposed by Dung, no translation exists from general, possibly \emph{non-flat} ABAFs into any kind of abstract argumentation formalism. 
		In this paper, we close this gap and show that bipolar AFs (BAFs) can instantiate general ABAFs. 
		To this end we develop suitable, novel BAF semantics which borrow from the notion of \emph{deductive support}. 
		We investigate basic properties of our BAFs, including computational complexity, and prove the desired relation to ABAFs under several semantics. 
	\end{abstract}
	
	\section{Introduction}	
	Computational models of argumentation~\cite{arguHandbook} play a central role in non-monotonic reasoning and have a wide range of applications in various fields such as %legal and medical reasoning
	law and healthcare~\cite{AtkinsonBGHPRST17}. One key aspect of these models is the use of  structured argumentation formalisms~\cite{struct-arg:2014}, which outline formal argumentative workflows from building blocks. 
	Prominent approaches include assumption-based argumentation (ABA)~\cite{BondarenkoDKT97}, ASPIC$^+$~\cite{ModgilP13}, %defeasible logic programming
	DeLP~\cite{DBLP:journals/tplp/GarciaS04}, and deductive argumentation~\cite{DBLP:books/daglib/0030834}. The reasoning process within these formalisms typically involves creating argument structures and identifying conflicts among them in a systematic manner from rule-based knowledge bases. 
	The resulting arguments and conflicts are known as argumentation frameworks (AFs)~\cite{Dung95}. These frameworks are then evaluated using %argumentation
	semantics to resolve conflicts,  determine the acceptability of arguments and draw conclusions  based on the original knowledge bases. 
	
	In this paper, we focus on ABA, as a versatile %modeling approach 
	structured argumentation formalism that, despite its simplicity, is able to handle reasoning with certain types of preferences, strict and defeasible rules, and different types of attacks without the need for additional tools \cite{Toni14}. Additionally, ABA can be naturally expanded to include more advanced reasoning with preferences \cite{CyrasT16} and probabilities~\cite{DungT10Prob}, can support applications (e.g. in healthcare~\cite{CravenTCHW12,CyrasOKT21}, law~\cite{DungTH10-law} and robotics~\cite{FanLZLM16}), and can be suitably deployed in multi-agent settings to support dialogues~\cite{FanT14}%, amongst others
	. 
	
	An ABA framework (ABAF) amounts to %sets 
 \FT{a set} of \emph{rules} from some deductive system,  candidate \emph{assumptions} amongst the sentences in its language, and a \emph{contrary}  for each assumption: at an abstract level, arguments are deductions supported by rules and assumptions, and attacks are directed at assumptions in the support of arguments, by means of arguments for their contraries. Thus, assumptions constitute the defeasible part of an ABAF.    
	There are no restrictions in ABA, in general, as to where the assumptions may appear in the rules~\cite{CyrasFST2018}.  
	A common restriction adopted in the study and deployment of ABA, however, is that %ABA knowledge bases are assumed to be 
	ABAFs are \emph{flat}, \ie each set of assumptions is \emph{closed}~\cite{BondarenkoDKT97}.  Intuitively speaking, flatness means that assumptions cannot be inferred, only assumed to be true or not. 
	Flat ABA is well-studied and, because of its relative simplicity~\cite{DBLP:journals/ai/CyrasHT21}, it is equipped with a variety of computational mechanisms~\cite{Toni13,ABAPlus,ABA-DL22}.
	However, general, \emph{non-flat} ABAFs have not yet been studied as comprehensively (except for aspects of computational complexity, e.g. as in~\cite{DBLP:journals/ai/CyrasHT21})% and, thus far, %the aforementioned tools are not available in this setting. 
	% to the best of our knowledge they lack computational tools
	, despite their potential for a broader range of applications than restricted flat ABA. The following example gives a simple illustration.   
	%Non-flat ABA can be naturally used to represent debates, as in the following example (for space reasons, this is a toy version of a possibly real debate about climate change).  Then, when using the semantics of non-flat ABA to assess the dialectical acceptability of assumptions, it is important to provide a high-level, abstract view of the ABA frameworks as \emph{explanations} for the debates' outcomes. \todo{FT to link better to the example and say why this needs to be non-flat ...hypothetico-deductive reasoning?}
	
	\begin{example}
		\label{ex:motivating}
		Consider the following discussion about climate change. 
		It is an abstraction of an idealised but realistic debate: is climate change actually happening? We cannot prove it for sure, so it makes sense to see it as an assumption %``climate change'' 
		(%assumption 
		$cc$), but we can try and establish it by looking at its consequences:  if it is actually happening we may expect an increased %amounts
  \FT{amount} of rain (assumption $mr$), %and even observe it ($mr \leftarrow$), 
		but then may need to deal with arguments against the validity of this assumption: one may argue that there has always been more rain at times, and so it is standard (assumption $sr$ for ``standard rain'') and thus object against $mr$ (using rule $not\_mr \leftarrow sr$), which in turn can be defeated by looking at statistics ($s$). This yields an ABAF $D$ consisting of 
		%literals 
		atoms $\mathcal{L} =  \{cc,mr,sr,s,not\_cc,not\_mr,not\_sr\}$, 
		assumptions $\mathcal{A} =  \{cc,mr,sr\}$, 
		and  rules ($\mathcal{R})$: 
		\begin{align*}
			%\mathcal{R}&= 
			&mr \leftarrow cc, 
			&&not\_mr \leftarrow sr, 
			&&not\_sr \leftarrow s, 
			&&s \leftarrow , 
			%&&mr \leftarrow  
			%&&s \leftarrow 
			%\mathcal{A} &=  \{cc,mr,r\}\\
			%\overline{X}&=not\_X, \forall X \in \mathcal{A}
		\end{align*}
		Moreover, for each assumption $X$, the \emph{contrary} is $not\_X$. 
		This is a non-flat ABAF, as  the assumption $mr$ is derivable from %(and thus supported by) 
		the assumption $cc$. 
		Allowing for assumptions to be derived from rules can thus accommodate a form of hypothetico-deductive reasoning  in debates of this form.	
	\end{example}
%	When considering flat ABAF, 
%	several computational methods (notably dispute derivations~\cite{Toni13}) greatly benefit from equivalent interpretations of the semantics of flat ABAFs as either sets of assumptions or sets of arguments in Dung-style AFs, empowering the use of the abstract computational notion of  \emph{dispute trees}~\cite{dung2006dialectic}. This mapping of flat ABAFs onto AFs also paves the way to 
%	explainability~\cite{fan2015computing,SchulzT16}. 
%	So, when considering general, non-flat ABAFs instead, a natural question arises as to whether they can be understood more abstractly and equipped with some form of dispute trees for explainability purposes.  

A main booster for the development %for 
\FT{of} flat ABAFs was the close correspondence to abstract AFs \cite{CyrasFST2018}. 
This contributed to the theoretical understanding of flat ABA, but plays also an important role in further aspects like explainability  \cite{DBLP:conf/ijcai/Cyras0ABT21}, dynamic environments \cite{DBLP:journals/jair/RapbergerU23}, and solving reasoning tasks \cite{DBLP:conf/kr/LehtonenR0W23}.

    In this paper, we extend this line of research and establish a connection between non-flat ABA and an abstract argumentation formalism. 
    To this end, we require two ingredients. 
    The first crucial observation is that ABAFs can be translated into \emph{bipolar} AFs (BAFs) \cite{KP01,cayrol2005acceptability,AmgoudCLL08} under a novel semantics. 
    As opposed to Dung-style AFs, BAFs do not only
    consider \FT{an} \emph{attack} %relationships 
    \FT{relation} between arguments, representing
	%that can represent contradictions
    conflicts, but also
    \FT{a} \emph{support} %relationships
    \FT{relation}, that can represent 
    justifications. Various semantics for BAFs have been proposed in the literature (see \cite{DBLP:journals/ijar/CayrolL13,BAFsurvey} for overviews).
    Our BAF semantics, which capture non-flat ABAFs, borrow ideas from previous approaches, but are novel in their technical details. 
    The second observation is that the aforementioned approach does not work for all common ABA semantics. 
    We tackle this issue by slightly extending our BAFs, similarly in spirit %with
    to so-called claim-augmented AFs (CAFs)~\cite{DvorakRW20kr}
    which assign to each argument a corresponding claim. 
    In our work, we will extend BAFs with \emph{premises} storing under which conditions an argument can be inferred. 
    
    The main contributions of this paper are as follows. 
	\begin{itemize}
		\item We define BAF semantics: novel, albeit similar in spirit to %previous approaches from the literature
		an existing semantics interpreting support as deductive~\cite{boella2010support}. We also study basic properties.%  of our novel semantics. 
		\item We show that for complete-based semantics,
  %BAFs under the novel semantics capture non-flat ABAFs.
  non-flat ABAFs admit a translation to BAFs w.r.t.\ our   semantics. 
            \item We propose so-called premise-augmented BAFs
            and show that they capture all common ABA semantics. 
		%\item We propose dispute trees which are suitably tailored to our BAF semantics. We demonstrate how one can extract argumentative explanations from our dispute trees. 
		\item We %discuss %the 
		analyse the
		computational complexity %of reasoning with 
		of our BAFs.  
	\end{itemize}
	%All proofs can be found in the Supplementary Material. 
	\section{Background}
	\label{sec:background}
	
	%\noindent\textbf{Abstract Argumentation.}\,
	\paragraph{Abstract Argumentation.}
	%\subsubsection{Abstract Argumentation}
	An \FT{abstract} argumentation framework (AF) \cite{Dung95} is a directed graph $\F = (A,\attacker)$ where $A$ represents a set of arguments and $\attacker\subseteq A\times A$ models \textit{attacks} between them. 
	%In this paper we assume $A$ to be a finite set. 
	%For a given $\F = (B,S)$ we let $A(\F) = B$ and $R(F) = S$.
	%For $U\subseteq A$ we define the restriction of $\F$ to $U$ as usual, i.e.\ $F\!\!\downarrow_U = (A\cap U , R\cap ({U\times U}))$.
	For two arguments $x,y\in A$, if $(x,y)\in \attacker$ we say that $x$ \textit{attacks} $y$\FT{, $x$ \textit{is an attacker of} $y$,} as well as $x$ \textit{attacks} (%the
 \FT{any} set) $E$ given that $y\in E\subseteq A$. 
	We let $E^+_F = \{ x\in A \mid E\text{ attacks }x \}$. % for a set $E\subseteq A$. 
	
	A set $E\subseteq A$ is \emph{conflict-free} in $\F$ %(for short, $E\in\cf(\F)$) 
	iff for no $x,y\in E$, $(x,y)\in \attacker$. 
	%We say 
	$E$ \textit{defends} an argument $x$ if 
	$E$ attacks each attacker of $x$.
	%any attacker of $x$ is attacked by some argument of $E$.
	A conflict-free set $E$ is \emph{admissible} in $\F$ ($E\in \adm(\F)$) iff it defends all its elements.
	%(we also say, $E$ defends itself).
	A \emph{semantics} is a function $F\mapsto\sigma(F)\subseteq 2^A$.
	This means, given an AF $\F = (A,R)$, a semantics returns a set of subsets of $A$. These subsets are called $\sigma$-\emph{extensions}.
	In this paper we consider so-called 
	%\textit{naive},
	\textit{admissible}, 
	\textit{complete}, 
	\textit{grounded},
	\textit{preferred}, 
	and \textit{stable}
	%\textit{semi-stable}
	%and \textit{stage}
	semantics (abbr. 
	%$\nav$, 
	$\adm$,
	$\com$, 
	$\grd$, 
	$\prf$,  
	$\stb$). 
	%\begin{definition} \label{def:extsem}
	For an AF and $E\in\adm(\F)$, we let %
	%\begin{enumerate}
	%\begin{itemize}
	%\item 
	i) $E\in\com(\F)$ iff $E$ contains all arguments it defends;
	%\item 
	ii) $E\in\grd(\F)$ iff $E$ is $\subseteq$-minimal in $\com(\F)$;
	%\item 
	iii) $E\in\prf(\F)$ iff $E$ is $\subseteq$-maximal in $\com(\F)$;
	%\item 
	iv) $E\in\stb(\F)$ iff $E_F^+=A\setminus E$.	
	%\end{itemize}
	%		$E\in \semi(\F)$ iff $E\in \com(\F)$ and there is no $D\in\com(\F)$ with $E\cup E^+_F\subset D\cup D^+_F$.
	%\end{enumerate}
	%\end{definition}
	
	A \emph{bipolar argumentation framework (BAF)} is a tuple
	$\CF = (\args, \attacker, \supporter)$, where $\args$ is a finite set of arguments, $\attacker \subseteq \args \times \args$ is the \emph{attack relation} as before and 
	$\supporter \subseteq \args \times \args$ is the \emph{support relation} \cite{AmgoudCLL08}. 
	Given a BAF $\CF = (\args, \attacker, \supporter)$\FT{,} we call $F = (\args, \attacker)$ the underlying AF of $\CF$. 
	Graphically, we depict the attack relation by solid edges and the support relation by dashed edges.	
	\paragraph{Assumption-based Argumentation.}
	We assume a deductive system $(\mathcal{L},\mathcal{R})$, where  $\mathcal{L}$ is a formal language, i.e., a set of sentences, 
	and $\mathcal{R}$ is a set of inference rules over $\mathcal{L}$. A rule $r \in \mathcal{R}$ has the form
	$a_0 \leftarrow a_1,\ldots,a_n$ with $a_i \in \mathcal{L}$.
	We denote the head of $r$ by $head(r) = a_0$ and the (possibly empty) 
	body of $r$ with $body(r) = \{a_1,\ldots,a_n\}$. 
	\begin{definition}
		An ABA framework is a tuple $(\mathcal{L},\mathcal{R},\mathcal{A},\contraryempty)$, where $(\mathcal{L},\mathcal{R})$ is a \emph{deductive system}, $\mathcal{A} \subseteq \mathcal{L}$ a non-empty set of \emph{assumptions}, and $\contraryempty:\mathcal{A}\rightarrow \mathcal{L}$ %the
		a \FT{(total)} \emph{contrary} function.
		%$\contraryempty$ is a function mapping assumptions $a\in \mathcal{A}$ to sentences $\mathcal{L}$. 
	\end{definition}

	We say that a sentence $p \in \mathcal{L}$ is \emph{%tree- YOU MOSTLY DO NOT USE IT, SO BEST TO KEEP IT SIMPLE
 \FT{derivable}} from assumptions $S \subseteq \mathcal{A}$ and rules $R \subseteq \mathcal{R}$, denoted by $S \vdash_R p$, if there is a finite rooted labeled tree $T$ such that the root is labeled with $p$, the set of labels for the leaves of $T$ is equal to $S$ or $S \cup \{\top\}$, and
	for every inner node $v$ of $T$ there is a rule $r \in R$ such that $v$ is labelled with $head(r)$, the number of successors of $v$
	is $|body(r)|$ and every successor of $v$ is
	labelled with a distinct $a \in body(r)$ or $\top$ if $body(r)=\emptyset$

	By $\theory_D(S)=\{p \in \mathcal L\mid \exists S'\subseteq S:S'\vdash_R p\}$
	we denote the set of all conclusions derivable from an assumption-set $S$ in an ABA \FT{framework (ABAF)} $D$. Observe that $S\subseteq \theory_D(S)$ since% per 
 \FT{, by} definition, each %assumption 
 $a\in \mathcal{A}$ is derivable from $\{a\}\vdash_\emptyset a$. 
	For %a set 
 $S\subseteq \mathcal A$, we let 
	$\contrary{S}=\{\contrary{a}\mid a\in S\}$; 
	moreover, for a %tree-
 \FT{derivation} $S\vdash p$ we write 
	$\asms(S\vdash p)=S$ 
	and for a set $E$ of %tree 
 \FT{derivations} we let
	$\asms(E)=\bigcup_{x\in E}\asms(x)$. 
 \FT{Also, we often write 
 $S \vdash_R p$ simply as $S \vdash p$.}
	
	A set $S\subseteq \mathcal A$ \emph{attacks} a set $T\subseteq \mathcal A$ if for some $a\in T$ we have that $\contrary{a}\in\theory_D(S)$.   
	A set $S$ is \emph{conflict-free}, denoted $E\in\cf(D)$, if it does not attack itself.
	With a little notational abuse we say $S$ attacks $a$ if $S$ attacks the singleton $\{a\}$.

	Given %a set $S$ of assumptions TO SAVE SPACE
 \FT{$S \!\subseteq \!\mathcal A$}, the \emph{closure} $\cl(S)$ of $S$ is %given as TO CUT SPACE
 $\cl(S) \!=\! \theory_D(S) \!\cap\! \mathcal A$. 
	With a little notational abuse we write $\cl(a)$ instead of $\cl(\{a\})$ whenever $S$ is a singleton. 
	A set $S\FT{\subseteq \mathcal A}$ %of assumptions TO CUT SPACE
 is 
 \emph{closed} if $S = \cl(S)$. 
	Observe that\FT{,}  in order for $S$ to be non-closed, it is necessary that $\mathcal R$ contains a rule 
	$a_0 \leftarrow a_1,\ldots,a_n$
	s.t.\ $a_0\in \mathcal A$, \ie the head of the rule is an assumption.

	Now we consider defense~\cite{BondarenkoDKT97,CyrasFST2018}.  Observe that defense in general ABAFs is only required against \emph{closed} sets of attackers. 
	Formally: 
	\begin{definition}
		\label{def:ABA defense}
		Let $D=(\mathcal{L},\mathcal{R},\mathcal{A},\contraryempty)$ be an ABA\FT{F}% framework
  , $S\subseteq \mathcal A$ and $a\in \mathcal A$. 
		We say that $S$ \emph{defends} $a$ iff for each closed set $T$ of assumptions s.t. $T$ attacks $a$, we have that $S$ attacks $T$; 
		$S$ defends itself iff $S$ defends each $b\in S$. 
	\end{definition}

 We next recall admissible, grounded, complete, preferred, and stable ABA semantics. %(abbr.\ $\adm$ $\grd$, $\comp$, $\prf$, $\stb$). 
	\begin{definition}
		\label{def:ABA semantics}
		Let $D=(\mathcal{L},\mathcal{R},\mathcal{A},\contraryempty)$ be an ABA% framework 
  \FT{F} and $S \subseteq \mathcal{A}$ be a set of assumptions s.t.\ $S\in\cf(S)$. We say
		\begin{itemize}
			\item $S\in\adm(D)$ iff $S$ is closed and defends itself;
			\item $S\in \prf(D)$ iff $S$ is $\subseteq$-maximal in $\adm(D)$;
			\item $S\in \comp(D)$ iff $S\in\adm(D)$ and %contains 
   \FT{is a superset of} every assumption set it defends; 
			\item $S\in \grd(D)$ iff $S = \bigcap_{T\in \com(D)} T$; 
			\item $S\in \stb(D)$ iff $S$ is closed and attacks each $x \in \mathcal{A} \setminus S$.
		\end{itemize}
		%	Moreover, 
		%	\begin{itemize}
			%		\item $S\in\sstb(D)$ iff $S$ is conflict-free, closed, and attacks $\cl(x)$ for each $x \in \mathcal{A} \setminus S$; 
			%		\item $S\in \cprf(D)$ iff $S$ is $\subseteq$-maximal in $\com(D)$. 
			%	\end{itemize}
	\end{definition}
	In this paper we stipulate that the empty intersection is interpreted as $\emptyset$, \ie if $\com(D) = \emptyset$, then 
	$\grd(D) = \emptyset$.
	\begin{example}
		\label{ex:aba bg example}
		Let 
		$D = (\mathcal{L},\mathcal{R},\mathcal{A},\contraryempty)$ 
		be the ABA% framework
  \FT{F} where 
		$\mathcal L = \{a,b,c,d,\contrary{a},\contrary{b},\contrary{c},\contrary{d}\}$, 
		$\mathcal A = \{a,b,c,d\}$, 
		the contrary function is given as indicated, and %rules 
  $\mathcal R$ \FT{consists of rules}: 
		\begin{align*}
			\contrary{b}\gets a. && \contrary{a}\gets b. && \contrary{d} \gets b.  && \contrary{b} \gets c. && d \gets c. 
		\end{align*}
		Let us discuss why $S = \{b\}$ is admissible in $D$. 
		First of all, $b$ is conflict-free as it does not %entail 
  \FT{derive} $\contrary{b}$. 
		Also, $b$ is closed, \ie $\cl(b) = \{b\}$. 
		Regarding defense, we have that 
		$\{a\}\vdash \contrary{b}$, but also  	
		$\{b\}\vdash \contrary{a}$, 
		so the attack is defended \FT{against}. 
		Finally, $c$ attacks $b$ ($\{c\}\vdash \contrary{b}$), but $\{c\}$ is not closed. Indeed, %the closure of $c$ is 
  $\cl(c) = \{c,d\}$. 
		Since $\{b\}\vdash \contrary{d}$, this attack is also defended \FT{against}. 
	\end{example}
	
	%%%%%%%%%%%%%%%%%%%%%%%%%%%%%%%%%%%%%%%%%%%%%%%
	\section{Closed Extensions for Bipolar AFs}
	\label{sec:BAFs}
	%%%%%%%%%%%%%%%%%%%%%%%%%%%%%%%%%%%%%%%%%%%%%%%
	Our goal is to translate non-flat ABA% frameworks 
 \FT{Fs} into BAFs. 
	In this section, we develop BAF semantics which are suitable for this endeavor (under complete-based semantics \FT{for ABAFs}). 
	To this end we interpret the support relation %similar in spirit to 
 \FT{in the spirit of} the notion of \emph{deductive support} \cite{boella2010support}, \ie the intuitive reading is that whenever $x$ is accepted and $x$ supports $y$, then $y$ is also accepted. 
	While this approach borrows from the BAF literature, to the best of our knowledge the exact definitions do not coincide with any previously proposed BAF semantics. 
	We define extension-based semantics directly on the given BAF, without re-writing it to an AF. 
	
	We start with the notion of the closure for BAFs. 
	\begin{definition}
		Let $\CF = (\args, \attacker, \supporter)$ be a BAF. 
		Consider the operator $\mu$ defined by $$\mu(E) = E\cup \{ a\in A \mid \exists e\in E:\; (e,a)\in\supporter \}.$$ 
		We call $\cl(E) = \bigcup_{n\geq 1} \mu^n(E)$ the \emph{closure} of $E$. 
		A set $E\subseteq \args$ is called \emph{closed} if $E = \cl(E)$.  
	\end{definition}
	Now we introduce the basic concepts of conflict-freeness and defense underlying our semantics. 
	As usual, a set of arguments is said to be conflict-free whenever it does not attack itself. 
	Our notion of defense is inspired by the way it is defined for ABA (cf.\ Definition~\ref{def:ABA defense}).  
	\begin{definition}
		\label{def_closure_baf}
		Let $\CF = (\args, \attacker, \supporter)$ be a BAF. 
		A set $E\subseteq A$ is \emph{conflict-free} if $E\cap E^+_{\FT{\mathcal{\CF}}} = \emptyset$; 
		$E$ \emph{defends} $a\in A$ if $E$ attacks each closed set $S\subseteq A$  which attacks $a$;
	the \emph{characteristic function} of $\CF$ is $\Gamma(E) = \{ a\in A\mid E\text{ defends }a \}$.% 
	\end{definition}
	Observe that this is a weaker condition than the defense notion of AFs since we can disregard non-closed attackers. 
	\begin{example}
		\label{ex:running example basics}
		Let $\CF$ be the following BAF (recall that the attack relation is depicted by solid edges and the support relation by dashed edges): 
		
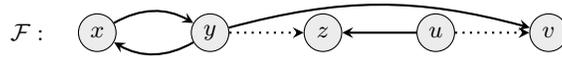
\begin{center}
			\begin{tikzpicture}[>=stealth,xscale=1, yscale=1]
				\small
				\path
				(0,0) node[arg](z){$z$}
				(-1.5,0) node[arg](y){$y$}
				(-3,0) node[arg,label=left:$\CF:\quad$](x){$x$}
				(1.5,0) node[arg](u){$u$}
				(3,0) node[arg](v){$v$}
				;
				\path [->,thick]
				(x) edge[bend left] (y)
				(y) edge[bend left] (x)
				(u) edge (z)
				(y) edge[bend left=15] (v)
				;
				\path [->,thick, dotted]
				(y) edge (z)
				(u) edge (v)
				;
			\end{tikzpicture}
		\end{center}
		We have that $\cl(y) = \{y,z\}$ 
		and $y$ defends $z$ which can be seen as follows: 
		%E
  \FT{e}ven though $u$ attacks $z$, $u$ is not a closed set of arguments. The closure of $\{u\}$ is $´\cl(\{u\}) = \{u,v\}$. 
		Since $y$ attacks $v$, we find $z\in\Gamma(\{y\})$. 
	\end{example}
	As we saw in this example, our defense notion can intuitively be interpreted \FT{as} follows: 
%I
\FT{i}f we seek to defend some argument $a$, then it suffices to counter-attack the closure of each attacker $b$ of $a$ (rather than $b$ itself). 
	\begin{lemma}
		\label{le:defense closure single arg}
		Let $\CF = (\args, \attacker, \supporter)$ be a BAF 
		and let $E\subseteq \args$ and $a\in \args$. 
		Then $E$ defends $a$ iff for each attacker $b$ of a it holds that $E$ attacks $\cl(\{b\})$. 
	\end{lemma}
	Let us now define admissibility. 
	We require a set of arguments to be conflict-free, closed, and self-defending. 
	\begin{definition}
		Let $\CF = (\args, \attacker, \supporter)$ be a BAF. 
		A set $E\subseteq A$ is \emph{admissible}, $E\in\adm(\CF)$, if  
		%\begin{itemize}
		%\item 
		i) $E$ is conflict-free, 
		%\item 
		ii) $E$ is closed, and 
		%\item 
		iii) $E\subseteq \Gamma(E)$. 
		%\end{itemize} 
	\end{definition}
	\begin{example}
		Recall Example~\ref{ex:running example basics}. 
		Let us verify that $E = \{y,z\}\in\adm(\CF)$. 
		Clearly, 
		$E$ is closed with $E\in\cf(\CF)$. 
		The two attackers of $E$ are $x$ and $u$ with 
		$\cl(x) = \{x\}$ 
		and 
		$\cl(u) = \{u,v\}$%;
  \FT{,}
		both of which are counter-attacked. 
		Another admissible set is $E' = \{u,v\}$ since the attack %of 
  \FT{by} $y$ is countered due to $\cl(y) = \{y,z\}$ and $u$ attacks $z$. 
	\end{example}
	As usual, the empty set is always admissible and hence, we can guarantee $\adm(\CF)\neq\emptyset$ for any given BAF $\CF$. 
	\begin{proposition}
		Let $\CF$ be a BAF. 
		Then $\emptyset\in\adm(\CF)$. 
		In particular, $\adm(\CF)\neq\emptyset$. 
	\end{proposition}
	Given this notion of admissibility, the definition of the remaining semantics is natural: 
	%F
 \FT{f}or complete extensions, we require $E$ to include all defended arguments; 
	preferred extensions are defined as maximal 
	%complete 
	admissible 
	sets; 
	the grounded extension is the intersection of all complete ones. 
	\begin{definition}
		Let $\CF = (\args, \attacker, \supporter)$ be a BAF. 
		A set $E\subseteq A$ of arguments s.t.\ $E\in\cf(\CF)$ is  
		\begin{itemize}
			\item \emph{preferred}, $E\in\prf(\CF)$, iff it is maximal admissible; 
			\item \emph{complete}, $E\in\com(\CF)$, iff $E\in\adm(\CF)$ and $E = \Gamma(E)$; 
			\item \emph{grounded}, $E\in\grd(\CF)$, iff $E = \bigcap_{S\in\com(\CF)}S$;
			\item \emph{stable}, $E\in\stb(\CF)$, iff it is closed and $E^+ = A\setminus E$. 
		\end{itemize}
	\end{definition}
	\begin{example}
		In our Example~\ref{ex:running example basics}, the admissible extension $E = \{y,z\}$ is maximal and thus preferred. 
		Moreover, $\{x,u,v\}\in\prf(\CF)$. 
		We observe however that $\{y,z\}$ is not complete since it does not contain the unattacked argument $u$. 
		Hence $\com(\CF) = \{\{ u,v \}\}$. 
	\end{example}
	%\todo{keep or remove the following?} KEEP!
	As a final remark regarding our BAF semantics, let us mention that they do not admit a translation into Dung-style AFs. 
	As the previous example already shows, preferred extensions are in general not complete (which is the case for AFs). 
	Another interesting observation is that we do not necessarily have a complete extension. 
	These properties show that a translation to AFs is impossible. 
	\begin{example}
		\label{ex:baf complete not necessary}
		Let $\CF$ be the following BAF: 
		\begin{center}
			\begin{tikzpicture}[>=stealth,xscale=1, yscale=1]
				\small
				\path
				(0,0) node[arg](z){$z$}
				(-1.5,0) node[arg](y){$y$}
				(-3,0) node[arg,label=left:$\CF:\quad$](x){$x$}
				;
				\path [->,thick]
				(x) edge (y)
				;
				\path [->,thick, dotted]
				(z) edge (y)
				;
			\end{tikzpicture}
		\end{center}
		Suppose $E\in\com(\CF)$. 
		Since $z$ is unattacked, we must have $z\in E$. 
		As complete sets must be closed, $y\in E$ follows. 
		However, by the same reasoning, $x\in E$ and thus, $E\notin\cf(\CF)$; a contradiction. 
		Indeed, the only admissible sets are $\emptyset$ and $\{x\}$%;
  \FT{,} both of which are not complete. 
	\end{example}
	Note that\FT{,} for the same reason, ABAFs cannot be translated into AFs: %G
 \FT{g}eneral ABAFs violate many properties that hold for AFs. Thus we require BAFs for the translation. 

	\section{%Instantiating Non-Flat ABA Frameworks}\todo{not quite the right title? we are also instantiating BAFs...May be: NonFlat ABAFs vs BAFs?}
 \FT{Instantiated BAFs}}
 
	\label{sec:ABA vs BAF}
	Suppose we are given an ABAF 
	$D = (\mathcal{L},\mathcal{R},\mathcal{A},\contraryempty)$. 
	Our goal is to translate $D$ into a BAF $\CF_D = (A,\attacker,\supporter)$. 
	The underlying idea is to define $A$ and $\attacker$ as it is done for flat ABAFs, \ie each ABA argument $S\vdash p$ corresponds to an argument in $\CF_D$  which attacks arguments in $\CF_D$ corresponding to $T\vdash q$ whenever $p\in \contrary{T}$. 
	What we have left to discuss is the support relation. 
	To this end 
	%suppose we accept all assumptions in $S$.  Then $T\subseteq S$ implies that we can also accept all arguments which rely on $T$. This naturally yields 
	%\begin{align}
	%	\label{eq:suppprter 1}
	%	T\subseteq S \; \Rightarrow \; (S\vdash p, T\vdash q)\in\supporter. 
	%\end{align}
	%In addition 
	we need to take care of the closure of %a 
 set\FT{s} of assumptions. 
	More specifically, if some assumption $a$ is in the closure of a set $S$, \ie $S\vdash a$ is %a tree-based 
 \FT{an} argument in $D$, then we encode this in $\CF_D$ %utilizing the $\supporter$ relation REPLACED TO SAVE SPACE
 \FT{using $\supporter$}. 
	
	To illustrate this, suppose we are given assumptions $a,b,c$ and rules 
	$r_1: p\gets a$, 
	$r_2: b\gets a$, and 
	$r_3: \contrary{b}\gets c$. 
	From $r_2$ it follows that $b\in\cl(a)$,  
	which can be encoded in our instantiated BAF as follows:  
	%S
 \FT{s}ince $b$ is an assumption, there is some generic argument $\{b\}\vdash b$ for it, then the argument stemming from rule $r_2$, \ie $\{a\}\vdash b$, supports the argument $\{b\}\vdash b$%. H
 \FT{; h}ence any closed set accepting $\{a\}\vdash b$ must also accept $\{b\}\vdash b$. 
	Including the usual attacks, this would give the following BAF (for $a\in\mathcal A$, we depict $\{a\}\vdash a$ by just $a$): 
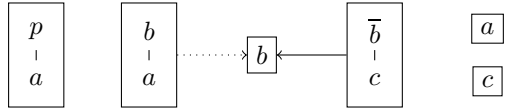
	\begin{center}
		\begin{tikzpicture}
			\node[draw] (arg1) at (0,0) {
				\begin{tikzpicture}[xscale=0.8,yscale=0.5]
					\node[targ] (p) at (0,0) {$b$};
					\node[targ] (q) at (0,-1.3) {$a$};
					
					\path[-]
					(p) edge (q)
					;
				\end{tikzpicture}
			};
			\node[draw] (arg2) at (-1.5,0) {
				\begin{tikzpicture}[xscale=0.8,yscale=0.5]
					\node[targ] (p) at (0,0) {$p$};
					\node[targ] (q) at (0,-1.3) {$a$};
					
					\path[-]
					(p) edge (q)
					;
				\end{tikzpicture}
			};
			\node[draw] (arg3) at (3,0) {
				\begin{tikzpicture}[xscale=0.8,yscale=0.5]
					\node[targ] (p) at (0,0) {$\contrary{b}$};
					\node[targ] (q) at (0,-1.3) {$c$};
					
					\path[-]
					(p) edge (q)
					;
				\end{tikzpicture}
			};
			
			\node[draw] (argb) at (1.5,0) {
				$b$
			};
			\node[draw] (arga) at (4.5,0.35) {
				$a$
			};
			\node[draw] (argc) at (4.5,-0.35) {
				$c$
			};
			
			\path[->]
			%			(arg1) edge[bend left] (arg2)
			%			(arg2) edge[bend left] (arg1)
			%			(arg1) edge (arg3)
			%			(arg4) edge (arg3)
			%			(arg4) edge[ bend left=45] (arg2)
			%			
			%			(arg2) edge (arga)
			%			(arg1) edge (argb)
			(arg3) edge (argb)
			%			(arg3) edge[bend right=20] (arg1)
			%			
			(arg1) edge[dotted] (argb)
			%			(arg5) edge[dotted] (argd)
			%			(argc) edge[dotted] (argd)
			;
			
		\end{tikzpicture}
	\end{center}
	It is now indeed impossible to accept $\{a\}\vdash b$ without counter-attacking $\{c\}\vdash \contrary{b}$ since the former supports $\{b\}\vdash b$. 
	With this support relation we miss however that we cannot accept $\{a\}\vdash p$, either: 
	%S
 \FT{s}ince constructing this argument requires $a$, we would then also have to include $b$ due to $b\in\cl(a)$. 
	Hence $\{a\}\vdash p$ should also support $\{b\}\vdash b$. 
	More generally, an argument $S\vdash p$ shall support each $a\in \cl(S)$: 
	\begin{align}
		\label{eq:suppprter 2}
		a\in\cl(S) \Rightarrow ( S\vdash p, \{a\} \vdash a )\in\supporter.
	\end{align}
	%It will turn out however that the supports specified in \eqref{eq:suppprter 1} are not required for the complete-based semantics (recall Proposition~\ref{prop:baf semantics adm vs co based}) which leads to a \emph{lean} instantiation which only construct the set \eqref{eq:suppprter 2} of supports.  
	We therefore define the support relation of our corresponding BAF according %to~\eqref{eq:suppprter 1} 
	to~\eqref{eq:suppprter 2} 
	as follows. 	
	\begin{definition}
		\label{def:inst baf}
		For an ABAF  
		$D = (\mathcal{L},\mathcal{R},\mathcal{A},\contraryempty)$, 
		we define the %corresponding
  \FT{instantiated} BAF 
		$\CF_D = (A,\attacker,\supporter)$ via 
		\begin{align*}
			A &= \{ (S\vdash p) \mid (S\vdash p)\text{ is  %a tree-based 
   \FT{an argument} in }D \}\\
			\attacker  &= \{ ( S\vdash p, T\vdash q ) \in A^2 \mid p \in\contrary{T} \}\\
			\supporter &= \{ ( S\vdash p, \{a\} \vdash a ) \in A^2 \mid a \in\cl(S) \} %\\
			%& \cup \{ (S\vdash p, T\vdash q) \mid T\subseteq S \}
		\end{align*} 
	\end{definition}
	
	\begin{example}	
		Recall our ABAF $D$ from Example~\ref{ex:aba bg example}. 
		The instantiated BAF $\CF_D$ is given as follows (\FT{again,} we depict the generic argument $\{a\}\vdash a$ for each $a\in \mathcal A$ by $a$). 
		%	To ease readability, we refrain from depicting the support arrows according to \eqref{eq:suppprter 1}, e.g.,\ $A_2$, $A_3$ and $b$ mutually support each other since they all rely on $\{b\}$. 
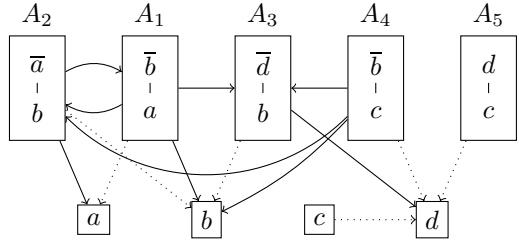
		\begin{center}
			\begin{tikzpicture}
				\node[draw,label={above:$A_1$}] (arg1) at (1.5,0) {
					\begin{tikzpicture}[xscale=0.8,yscale=0.5]
						\node[targ] (p) at (0,0) {$\contrary{b}$};
						\node[targ] (q) at (0,-1.3) {$a$};
						
						\path[-]
						(p) edge (q)
						;
					\end{tikzpicture}
				};
				\node[draw,label={above:$A_2$}] (arg2) at (0,0) {
					\begin{tikzpicture}[xscale=0.8,yscale=0.5]
						\node[targ] (p) at (0,0) {$\contrary{a}$};
						\node[targ] (q) at (0,-1.3) {$b$};
						
						\path[-]
						(p) edge (q)
						;
					\end{tikzpicture}
				};
				\node[draw,label={above:$A_3$}] (arg3) at (3,0) {
					\begin{tikzpicture}[xscale=0.8,yscale=0.5]
						\node[targ] (p) at (0,0) {$\contrary{d}$};
						\node[targ] (q) at (0,-1.3) {$b$};
						
						\path[-]
						(p) edge (q)
						;
					\end{tikzpicture}
				};
				\node[draw,label={above:$A_4$}] (arg4) at (4.5,0) {
					\begin{tikzpicture}[xscale=0.8,yscale=0.5]
						\node[targ] (p) at (0,0) {$\contrary{b}$};
						\node[targ] (q) at (0,-1.3) {$c$};
						
						\path[-]
						(p) edge (q)
						;
					\end{tikzpicture}
				};
				\node[draw,label={above:$A_5$}] (arg5) at (6,0) {
					\begin{tikzpicture}[xscale=0.8,yscale=0.5]
						\node[targ] (p) at (0,0) {$d$};
						\node[targ] (q) at (0,-1.3) {$c$};
						
						\path[-]
						(p) edge (q)
						;
					\end{tikzpicture}
				};

				\node[draw] (arga) at (0.75,-1.75) {
					$a$
				};
				\node[draw] (argb) at (2.25,-1.75) {
					$b$
				};
				\node[draw] (argc) at (3.75,-1.75) {
					$c$
				};
				\node[draw] (argd) at (5.25,-1.75) {
					$d$
				};
				
				\path[->]
				(arg1) edge[bend left] (arg2)
				(arg2) edge[bend left] (arg1)
				(arg1) edge (arg3)
				(arg4) edge (arg3)
				(arg4) edge[ bend left=45] (arg2)
				
				(arg2) edge (arga)
				(arg1) edge (argb)
				(arg3) edge (argd)
				(arg4) edge[bend left = 10] (argb)
				
				(arg1) edge[dotted] (arga)
				(arg2) edge[dotted] (argb)
				(arg3) edge[dotted] (argb)
				
				(arg4) edge[dotted] (argd)
				(arg5) edge[dotted] (argd)
				(argc) edge[dotted] (argd)
				;
				
			\end{tikzpicture}
		\end{center}
		%Observe that this almost coincides with the usual ABA instantiation as a Dung-style AF, with only some supports added, which denote the closure of sets of arguments.
		As we saw, $\{b\}$ is admissible in $D$. Now consider the set $E$ of all arguments with $\{b\}$ as assumption set, \ie  
		$E = \{A_2,A_3,b\}$. 
		Again, $E$ is conflict-free and closed (in $\CF_D$). 
		The attack from $A_1$ is countered; moreover, $A_4$ attacks $A_3$%.		H
  \FT{; h}owever, as $A_4$ supports $d$, the closure of $A_4$ (in $\CF_D$) is $\{A_4,d\}$ with $A_3$ attacking $d$; so this attack is also countered. 
		We infer that $E$ is admissible in $\CF_D$. 
	\end{example}
	Though Definition~\ref{def:inst baf} induces infinitely many arguments, they are determined by their underlying assumptions and conclusion. 
	Hence it suffices to construct a finite BAF. 

	In the remainder of this section, we establish the following main result showing that the BAF $\CF_D$ as defined above is suitable to %instantiate 
 \FT{capture} non-flat ABAFs for complete-based semantics: 
   % More specifically, 
    if $E$ is some extension (in $\CF_D$), then a set of acceptable assumptions can be obtained by gathering all assumptions underlying the arguments in $E$; 
    and if $S$ is acceptable (in $D$), then all arguments constructible from the assumptions in $S$ from a corresponding extension in the BAF $\CF_D$.% 
	\begin{restatable}{theorem}{thSemanticsCorrespondence}
		\label{th:semantics correspondence}
		Let $D = (\mathcal{L},\mathcal{R},\mathcal{A},\contraryempty)$ be an ABAF 
		and $\CF_D = (A,\attacker,\supporter)$ the %corresponding 
  \FT{instantiated} BAF. 
		Let $\sigma\in\{\com,\grd,\stb\}$. 
		\begin{itemize}
			\item If $E\in\sigma(\CF_D)$, then $\asms(E)\in\sigma(D)$. 
			\item If $S\in\sigma(D)$, then $\{ x\in A \mid \asms(x)\subseteq S \}\in\sigma(\CF_D)$.
		\end{itemize}
		%Moreover, 
		%\begin{itemize}
		%\item 
		If $S\in\adm(D)$, then $\{ x\in A \mid \asms(x)\subseteq S \}\in\adm(\CF_D)$.
		%\end{itemize}
	\end{restatable}
    	\begin{example}
		Recall our ABA $D$ and $\CF_D$ from above. 
		As we already saw, $S = \{b\}\in\adm(D)$ 
		and indeed, $E = \{x\in A \mid \asms(x)\subseteq S\} = \{A_2,A_3,b\}\in\adm(\CF_D)$ as we verified. 
	\end{example}
    As stated in the theorem, we only get one direction for \FT{the} $\adm$ semantics; for \FT{the} $\prf$ semantics, both directions fail.
    %The reason is that admissible-based semantics are not \emph{exhaustive}, i.e., not all necessary arguments are considered automatically.
    We will discuss and subsequently
    %We will work out why this construction fails in this case and 
    fix the underlying issue later. 
	%	Next we devote a subsection to each item in Theorem~\ref{th:semantics correspondence}. 
	%	We start with the second one, \ie given an extension $S\in\sigma(D)$ we show that one can construct a corresponding $\sigma$-extension in $\CF_D$. 
	\subsection{From ABA to BAF}
	Translating an extension of the given ABAF into one in the BAF is the easier direction. 
	Our first step is the following proposition which shows the desired connection between conflict-free and closed sets. %From this we can entail that admissibility is preserved. 
	\begin{restatable}{proposition}{propABAToAFAdm}
		\label{prop:ABA to AF adm}
		Let $D = (\mathcal{L},\mathcal{R},\mathcal{A},\contraryempty)$ be an ABAF 
		and $\CF_D = (A,\attacker,\supporter)$ the %corresponding 
  \FT{instantiated} BAF. 
		Let $S\subseteq \mathcal A$ and let $E = \{ x\in A \mid \asms(x)\subseteq S \}$. 
		\begin{itemize}
			\item If $S\in\cf(D)$, then $E\in\cf(\CF_D)$. 
			\item If $S$ is closed in $D$, then $E$ is closed in $\CF_D$. 
			\item If $S\in\adm(D)$, then $E\in\adm(\CF_D)$. 
		\end{itemize}
	\end{restatable}
	%\begin{proof}[Sketch of Proof]
	%	The first two items follow with straightforward arguments. We show how defense survives the transition to $\CF_D$. 
	%	Suppose $E'$ is a closed set of arguments attacking $E$. 
	%	Let us first assume $E'$ is of the form $E' = \cl(x)$ for some $x\in A$. 
	%	Let $T = \asms(x)$. Then $\cl(T) = \asms(E')$ can be seen. % by Lemma~\ref{le:closure of single arguments}. 
	%	By admissibility of $S$, $S\vdash \bar a$ for some $a\in \cl(T)$, \ie some $a\in\asms(E')$. 
	%	By definition, there is some tree-based argument $y = S'\vdash \bar a$ with $S'\subseteq S$. 
	%	By choice of $E$ we have $y\in E$ due to $\asms(y) = S'\subseteq S$. 
	%	Hence $E$ attacks $E'$ in $\CF_D$, \ie $E$ defends itself against $E'$ as desired. 
	%	Now for the general case suppose $E'$ is an arbitrary closed set of arguments attacking $E$. 
	%	Observe that for each $x\in E'$, $\cl(x)\subseteq \cl(E')$. 
	%	Hence take any $x\in E'$ attacking $E$. By the above reasoning, $E$ attacks $\cl(x)$ and thus $E$ attacks $E'$. 
	%\end{proof}
	In order to extend this result to complete extensions, we have only left to show that all defended arguments are included in the corresponding BAF $\CF_D$. 
	%This is indeed the case.%, as we show next. 
	\begin{restatable}{proposition}{propABAToAFCom}
		\label{prop:ABA to AF com}
		Let $D = (\mathcal{L},\mathcal{R},\mathcal{A},\contraryempty)$ be an ABAF 
		and $\CF_D = (A,\attacker,\supporter)$ the %corresponding 
  \FT{instantiated} BAF. 
		If $S\in\com(D)$, then for $E = \{ x\in A \mid \asms(x)\subseteq S \}$ we get $E\in\com(\CF_D)$. 
	\end{restatable}
	%	\begin{proof}
		%		Since admissibility is already established, we have left to show: 
		%		
		%		(fixed-point) 
		%		Suppose $E$ defends $x\in A$. 
		%		We have to show that $x\in E$. 	
		%		Let $a\in\asms(x)$. 
		%		We show that $a$ is defended by $S$ in $D$. 
		%		We make use of Lemma~\ref{le:reduce defense to single arg} and consider some tree-based argument $T\vdash \bar a$, \ie it attacks $a$. 
		%		
		%		Now $T\vdash \bar a$ attacks $x$ in $\CF_D$. 
		%		Since $E$ defends $x$, there is some $S'\subseteq S = \asms(E)$ with $S'\vdash \bar b$ for some $b\in \cl(T)$; in particular, this argument $S'\vdash \bar b$ is contained in $E$.  
		%		Thus, $\bar b\in \theory_D(S)$ and therefore $S$ counter-attacks $\cl(T)$ in $D$. 
		%		
		%		As $T\vdash \bar a$ was an arbitrary attacker of $a$, Lemma~\ref{le:reduce defense to single arg} ensures that $S$ defends $a$. 
		%		Completeness of $S$ thus implies $a\in S$. 
		%		Since $a$ was an arbitrary assumption in $\asms(x)$, we deduce $\asms(x)\subseteq S$. 
		%		By construction of $E$, $x\in E$ as desired.  
		%	\end{proof}
	Moreover, the set of attacked assumptions is preserved. 
	Thus, we also find stable 
	%and set-stable 
	extensions of $\CF_D$.% 
	\begin{restatable}{proposition}{propABAToAFStb}
		\label{prop:ABA to AF stb}
		Let $D = (\mathcal{L},\mathcal{R},\mathcal{A},\contraryempty)$ be an ABAF 
		and $\CF_D = (A,\attacker,\supporter)$ the %corresponding 
  \FT{instantiated} BAF. 
		%Let $\sigma\in\{\stb,\sstb\}$.  
		If $S\in\stb(D)$, then for $E = \{ x\in A \mid \asms(x)\subseteq S \}$ we get $E\!\in\!\stb(\CF_D)$. 
	\end{restatable}
	%	\begin{proof}
		%		We know already that $E$ is conflict-free and closed. 
		%		
		%		($\stb$) 
		%		Suppose $x\in A\setminus E$. 
		%		Let $x$ be of the form $T\vdash p$. 
		%		By construction of $E$, $T\setminus S\neq \emptyset$; consider some $a\in T\setminus S$. 
		%		Since $S$ is stable, $S$ attacks $a$, \ie there is some tree-based argument $S'\vdash \bar a$ with $S'\subseteq S$. 
		%		We have $S'\vdash \bar a\in E$ and hence $E$ attacks $x$. 
		%		
		%		($\sstb$) 
		%		Again let $x\in A\setminus E$ be of the form $T\vdash p$. 
		%		As above, $T\setminus S\neq \emptyset$ and we consider some $a\in T\setminus S$. 
		%		Since $S$ is set-stable, $S$ attacks $\cl(a)$, 
		%		Hence there is some $b\in\cl(a)$ and $S'\vdash \contrary{b}$ a tree-based argument where $S'\subseteq S$.  
		%		In particular, $S'\vdash \contrary{b}\in E$. 
		%		By definition of the support relation $\supporter$, $x = T\vdash p$ supports $\{b\}\vdash b$ due to $b\in\cl(T)$ (because $a\in T$). 
		%		Hence $E$ attacks $\cl(x)$. 
		%	\end{proof}
	Consequently, the first item in Theorem~\ref{th:semantics correspondence} is shown. 
	
	\subsection{From BAF to ABA}
	Turning extensions of the instantiated BAF into extensions of the underlying ABAF is more involved and does not work for admissible sets without any further restriction. 
	It is important to understand \emph{why} it does not work for admissible sets since this will also demonstrate why complete-based semantics do not face this issue. 
	The problem is related to the way we have to construct our support relation. 
	We illustrate this in the following example.% 
	\begin{example}
		\label{ex:nont assumption exhaustive mismatch}
		Let 
		$D = (\mathcal{L},\mathcal{R},\mathcal{A},\contraryempty)$ 
		be the ABAF where 
		$\mathcal L = \{a,b,c,\contrary{a},\contrary{b},\contrary{c},p\}$, 
		$\mathcal A = \{a,b,c\}$, 
		$\contraryempty$ is as indicated, and %rules 
  $\mathcal R = \{p\gets a.,\; q\gets b.,\; c \gets p,q.,\; \contrary{c} \gets c. \}$.
		Observe that $S=\{a,b\}$ is not admissible in $D$ since $S$ is not closed: %I
  \FT{i}ndeed, we can derive 
		$p$ from $a$ and 
		$q$ from $b$ and thus $c$ from $S$, \ie $c\in\cl(S)$. 
		Now consider $\CF_D = (A,\attacker,\supporter)$:
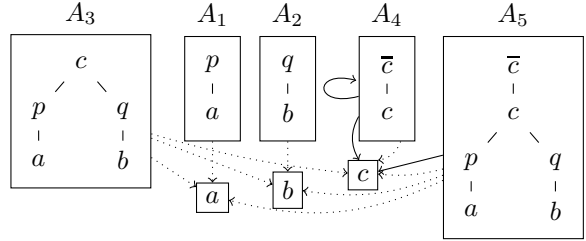
		\begin{center}
			\begin{tikzpicture}
				\node[draw,label={above:$A_1$}] (arg1a) at (2.5,0.65) {
					\begin{tikzpicture}[xscale=0.8,yscale=0.5]
						\node[targ] (p) at (0,0) {$p$};
						\node[targ] (q) at (0,-1.3) {$a$};
						
						\path[-]
						(p) edge (q)
						;
					\end{tikzpicture}
				};
				\node[draw,label={above:$A_2$}] (arg1b) at (3.5,0.65) {
					\begin{tikzpicture}[xscale=0.8,yscale=0.5]
						\node[targ] (p) at (0,0) {$q$};
						\node[targ] (q) at (0,-1.3) {$b$};
						
						\path[-]
						(p) edge (q)
						;
					\end{tikzpicture}
				};
				\node[draw,label={above:$A_3$}] (arg2) at (0.75,0.35) {
					\begin{tikzpicture}[xscale=0.8,yscale=0.5]
						\node[targ] (c) at (0,1.3) {$c$};
						\node[targ] (p) at (-.7,0) {$p$};
						\node[targ] (q) at (.7,0) {$q$};
						\node[targ] (a) at (-.7,-1.3) {$a$};
						\node[targ] (b) at (.7,-1.3) {$b$};
						
						\path[-]
						(a) edge (p)
						(b) edge (q)
						(p) edge (c)
						(q) edge (c)
						;
					\end{tikzpicture}
				};
				\node[draw,label={above:$A_4$}] (arg4) at (4.82,0.65) {
					\begin{tikzpicture}[xscale=0.8,yscale=0.5]
						\node[targ] (c) at (0,0) {$c$};
						\node[targ] (nc) at (0,1.3) {$\contrary{c}$};
						
						\path[-]
						(c) edge (nc)
						;
					\end{tikzpicture}
				};
				\node[draw,label={above:$A_5$}] (arg3) at (6.5,0) {
					\begin{tikzpicture}[xscale=0.8,yscale=0.5]
						\node[targ] (nc) at (0,2.6) {$\contrary{c}$};
						\node[targ] (c) at (0,1.3) {$c$};
						\node[targ] (p) at (-.7,0) {$p$};
						\node[targ] (q) at (.7,0) {$q$};
						\node[targ] (a) at (-.7,-1.3) {$a$};
						\node[targ] (b) at (.7,-1.3) {$b$};
						
						\path[-]
						(a) edge (p)
						(b) edge (q)
						(p) edge (c)
						(q) edge (c)
						(c) edge (nc)
						;
					\end{tikzpicture}
				};
				
				\node[draw] (arga) at (2.5,-.8) {
					$a$
				};
				\node[draw] (argb) at (3.5,-.7) {
					$b$
				};
				\node[draw] (argc) at (4.5,-.5) {
					$c$
				};
				
				\path[->]
				(arg4) edge[loop left] (arg4)
				(arg4) edge[bend right] (argc)
				(arg3) edge (argc)
				
				(arg1a) edge[dotted] (arga)
				(arg1b) edge[dotted] (argb)
				%die von A_3: 
				(arg2) edge[dotted] (arga)
				(arg2) edge[dotted] (argb)
				%die von A_5:
				(arg3) edge[dotted, bend left=20] (arga)
				(arg3) edge[dotted, bend left=15] (argb)
				(arg4) edge[bend left, dotted] (argc)
				(arg3) edge[dotted, bend left=10] (argc)
				(arg2) edge[dotted, bend right=5] (argc)
				;
			\end{tikzpicture}
		\end{center}
		We want to emphasize that there is neither a support arrow from $A_1$ to $c$ nor from $A_2$ to $c$; the fact that 
		$c\in\cl(\{a,b\})$  
		holds is reflected in 
		$(A_3,c)\in\supporter$ 
		and 
		$(A_5,c)\in\supporter$. 
		
		Consider now %the admissible extension
		$E = \{a,b,A_1,A_2\}.$ 
		As all arguments in $E$ are unattacked and have no out-going support arrows, it is clear that $E\in\adm(\CF_D)$.
		Yet, the required assumptions to build these arguments are $\{a,b\}$, despite $\{a,b\}\notin\adm(D)$.
		%Trying to fix this issue by building the closure, we find $\cl(\{a,b\}) = \{a,b,c\}$ with $c$ being a self-attacker. 
		%That is, we also cannot extend $\{a,b\}$ to an admissible set in $D$.  
	\end{example}
	The mismatch in the previous example occurred because we did not take \emph{all} arguments we can build from $a$ and $b$. 
	Indeed%: W
 \FT{, w}e did not include $A_3$ and $A_5$ in our extension $E$ of $\CF_D$. 
	These arguments encode the support from $a$ and $b$ to $c$ and\FT{,} thus, we would have detected the missing $c$ we cannot defend. 
	This observation leads to the following notion. 
	\begin{definition}
		Let $D = (\mathcal{L},\mathcal{R},\mathcal{A},\contraryempty)$ be an ABAF;
        $\CF_D = (A,\attacker,\supporter)$ the %associated 
        \FT{instantiated} BAF. 
		A set $E\subseteq A$ is \emph{assumption exhaustive} if $\asms(x)\subseteq \asms(E)$ implies $x\in E$.
	\end{definition}
	\begin{example}
		In the previous Example~\ref{ex:nont assumption exhaustive mismatch}, 
		for the set $E = \{a,b,A_1,A_2\}$ of arguments we have $\asms(E) = \{a,b\}$ and hence $E$ is not assumption exhaustive because $A_3$ and $A_5$ also satisfy $\asms(A_i)\subseteq \asms(E)$. 
	\end{example}
	\begin{remark}
		In the previous subsection we started with assumptions $S$ and constructed $E = \{x\in A\mid \asms(x)\subseteq S \}$. 
		Such $E$ is assumption exhaustive by design. 
	\end{remark}
	The next proposition states that the problematic behavior 
	we observed in Example~\ref{ex:nont assumption exhaustive mismatch} regarding admissible extensions 
	does not occur for assumption exhaustive sets.% of arguments. 
	\begin{restatable}{proposition}{propAssumptionExhaustiveAdm}
		\label{prop:assumption exhaustive adm}
		Let $D = (\mathcal{L},\mathcal{R},\mathcal{A},\contraryempty)$ be an ABAF 
		and $\CF_D = (A,\attacker,\supporter)$ the %corresponding 
  \FT{instantiated} BAF. 
		Let $E\subseteq A$ be assumption exhaustive and let $S = \asms(E)$.
		\begin{itemize}
			\item If $E\in\cf(\CF_D)$, then $S\in\cf(D)$. 
			\item If $E$ is closed in $\CF_D$, then $S$ is closed in $D$. 
			\item If $E\in\adm(\CF_D)$, then $S\in\adm(D)$. 
		\end{itemize}
	\end{restatable}
	%	\begin{proof}
		%		(conflict-free) 
		%		Suppose $S\vdash \bar a$ with $a\in S$. 
		%		Then, there is a tree-based argument $x = S'\vdash \bar a$ with $S'\subseteq S$ in $\CF_D$. 
		%		By Lemma~\ref{le:closed contains asms}, $x$ attacks $E$ in $\CF_D$. 
		%		Since $S'\subseteq S = \asms(E)$, we infer $x\in E$ since $E$ is assumption exhaustive. 
		%		Hence $E$ is not conflict-free, a contradiction. 
		%		
		%		(closed) 
		%		Suppose $a\in \cl(S)$. 
		%		Then there is some $S'\subseteq S$ with $S'\vdash a$. 
		%		Since $E$ is assumption exhaustive, $E$ contains each argument $x\in A$ with $\asms(x)\subseteq S$; in particular we obtain $S'\vdash a\in E$. 
		%		Since $(S'\vdash a,\{a\}\vdash a)\in\supporter$ and $E$ is closed, $\{a\}\vdash a\in E$ follows. 
		%		Hence $a\in\asms(E) = S$. 
		%		
		%		(defense)
		%		Let $S'$ be a closed set of assumptions attacking $S$. 
		%		Applying Lemma~\ref{le:reduce defense to single arg} we suppose $S' = \cl(T)$ for some tree-based argument $x = T\vdash \bar a$ with $a\in S$. 
		%		Since $a\in\asms(E)$, and $E$ is assumption exhaustive, $x$ attacks $E$ in $\CF_D$. 
		%		By admissibility of $E$, there is some $S'\subseteq S$ and a tree-based argument $S'\vdash \bar b\in E$ s.t.\ $b\in \cl(T)$. 
		%		Since $S'\subseteq S$, we have $\bar b\in\theory_D(S)$ as well and hence, $S$ counter-attacks $S'$ in $D$. 	
		%	\end{proof}
	Given the results we have established so far, Proposition~\ref{prop:assumption exhaustive adm} can only serve as an intermediate step because it relies on assumption exhaustive sets of arguments. 
	However, within the context of an abstract BAF this notion does not make any sense; it is tailored %for 
 \FT{to} arguments stemming from instantiating $D$. 	
	Because of this, the following lemma is crucial. 
	It states that for complete-based semantics 
	%(cf.\ Proposition~\ref{prop:baf semantics adm vs co based}), 
	each extension is assumption exhaustive. 
	Thus, in this case, the mismatch we observed in Example~\ref{ex:nont assumption exhaustive mismatch} does not occur. 
	\begin{restatable}{lemma}{leCoContainsAllAsms}
		\label{le:co contains all asms}
		Let $D = (\mathcal{L},\mathcal{R},\mathcal{A},\contraryempty)$ be an ABAF 
		and $\CF_D = (A,\attacker,\supporter)$ the %corresponding 
  \FT{instantiated} BAF. 
		If $E\in\com(\CF_D)$, then $E$ is assumption exhaustive. 
	\end{restatable}
	%	\begin{proof}
		%		Let $Y\subseteq A$ be a closed set of arguments attacking $x = S \vdash p$ with $S\subseteq \asms(E)$. 
		%		Then some $y \in Y$ is of the form $T\vdash \bar a$ with $a\in S$. 
		%		Since $a\in\asms(E)$, $y$ attacks $E$ by Lemma~\ref{le:closed contains asms}. 
		%		Hence $E\to y$ and thus $E\to Y$.  
		%		Since $Y$ was an arbitrary closed attacker of $x$, $E$ defends $x$. 
		%		Since $E$ is complete, $x\in E$. 
		%	\end{proof}
	\begin{example}
		In Example~\ref{ex:nont assumption exhaustive mismatch},  we saw that the set 
		$E = \{a,b,A_1,A_2\}$ of arguments is not assumption exhaustive. 
		Indeed, since $E$ is admissible, it is clear by definition that each argument $x$ with $\asms(x)\subseteq \{a,b\}$ is defended, including $A_3$ and $A_5$. 
		Hence $E$ is not complete. 
		On the other hand,  
		$\Gamma(E) =  \{a,b,A_1,A_2,A_3,A_5\} = E'$  
		is assumption exhaustive as desired.   
		%Furthermore, we now have $c\in\cl(E')$ which cannot be defended and\FT{,} thus, there is no complete extension in $\CF_D$ which corresponds to $\{a,b\}$. 
	\end{example}
	Lemma~\ref{le:co contains all asms} allows us to apply 
	Proposition~\ref{prop:assumption exhaustive adm} 
	to all complete extensions. 
	Hence we can infer the following result without restricting to assumption exhaustive sets.  %explicitly. 
	\begin{restatable}{proposition}{propBAFToABACom}
		\label{prop:BAF to ABA com}
		Let $D = (\mathcal{L},\mathcal{R},\mathcal{A},\contraryempty)$ be an ABAF 
		and $\CF_D = (A,\attacker,\supporter)$ the %corresponding 
  \FT{instantiated} BAF. 
		If $E\in\com(\CF_D)$, then $S = \asms(E)\in\com(D)$. 
	\end{restatable}
	%	\begin{proof}	
		%		(fixed-point) 
		%		Suppose $S$ defends $a\in\mathcal A$ in $D$. We have to show $a\in S$. 
		%		For this, it suffices to show that $E$ defends $\{a\}\vdash a$, because this implies $\{a\}\vdash a\in E$ (completeness) and hence $a\in\asms(E) = S$. 
		%		Thus, let $E'$ be closed set of arguments in $\CF_D$ attacking $\{a\}\vdash a$. 
		%		
		%		We first assume $E' = \cl(T)$ for some $x = T\vdash \bar a$. 
		%		By assumption, $S$ defends $a$ and thus, $\bar b\in\theory(S)$ for some $b\in T$. 
		%		Hence there is some tree-based argument $S'\vdash \bar b$ with $S'\subseteq S$. 
		%		Since $S'\subseteq S = \asms(E)$, we have that $S'\vdash \bar b\in E$ by Lemma~\ref{le:co contains all asms}. 
		%		It follows that $E$ counter-attacks $E'$. 
		%		In case $E'$ is an arbitrary closed set of attackers, take one argument $x\in E'$ attacking $E$ and reason analogously. 
		%		
		%		Since $E'$ was an arbitrary closed attacker of $\{a\}\vdash a$ and since $E$ is complete, $\{a\}\vdash a\in E$ as desired. 
		%		We deduce $a\in S$ and again since $a$ was an arbitrary assumption defended by $S$, we have that $S$ contains all defended assumptions in $D$. 
		%	\end{proof}
	%As a corollary, we entail the claim from Theorem~\ref{th:semantics correspondence} for %$\grd$ 
	%and $\cprf$ 
	%extensions, since they intersect 
	%and maximize 
	%$\com$. 
	%, respectively. 
	\begin{corollary}
		Let $D = (\mathcal{L},\mathcal{R},\mathcal{A},\contraryempty)$ be an ABAF 
		and $\CF_D = (A,\attacker,\supporter)$ the %corresponding 
  \FT{instantiated} BAF. 
		%Let $\sigma\in\{\grd,\cprf\}$. 
		\begin{itemize}
			\item If $E\in\grd(\CF_D)$, then $S = \asms(E)\in\grd(D)$. 
			\item If $S\in\grd(D)$, then for $E = \{ x\in A \mid \asms(x)\subseteq S \}$ we have that $E\in\grd(\CF_D)$. 
		\end{itemize}
	\end{corollary}
	Finally, stable extensions (in the BAF $\CF_D$)	are also assumption exhaustive due to Lemma~\ref{le:co contains all asms} (each stable extension is complete). 
	This yields the desired connection for $\stb$.% 
	\begin{restatable}{proposition}{propBAFToABAStb}
		\label{prop:BAF to ABA stb}
		Let $D = (\mathcal{L},\mathcal{R},\mathcal{A},\contraryempty)$ be an ABAF 
		and $\CF_D = (A,\attacker,\supporter)$ the %corresponding 
  \FT{instantiated} BAF. 
		If $E\in\stb(\CF_D)$, then $S = \asms(E)\in\stb(D)$. 
	\end{restatable}
\begin{example}
	Let us head back to our motivating Example~\ref{ex:motivating} on climate change. 
	We instantiate the following BAF. 
	
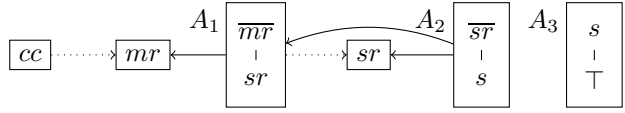
\begin{center}
		\begin{tikzpicture}[xscale=1.5]
			\node[] (labl) at (2.55,-1.3) {$A_1$};
			\node[draw] (argl) at (3,-1.75) {
				\begin{tikzpicture}[xscale=0.8,yscale=0.5]
					\node[targ] (p) at (0,0) {$\contrary{mr}$};
					\node[targ] (q) at (0,-1.3) {$sr$};
					
					\path[-]
					(p) edge (q)
					;
				\end{tikzpicture}
			};
			\node[] (labl) at (4.55,-1.3) {$A_2$};
			\node[draw] (argm) at (5,-1.75) {
				\begin{tikzpicture}[xscale=0.8,yscale=0.5]
					\node[targ] (p) at (0,0) {$\contrary{sr}$};
					\node[targ] (q) at (0,-1.3) {$s$};
					
					\path[-]
					(p) edge (q)
					;
				\end{tikzpicture}
			};
			\node[] (labl) at (5.55,-1.3) {$A_3$};
			\node[draw] (argr) at (6,-1.75) {
				\begin{tikzpicture}[xscale=0.8,yscale=0.5]
					\node[targ] (p) at (0,0) {$s$};
					\node[targ] (q) at (0,-1.3) {$\top$};
					
					\path[-]
					(p) edge (q)
					;
				\end{tikzpicture}
			};

			\node[draw] (cc) at (1,-1.75) {
				$cc$
			};
			\node[draw] (mr) at (2,-1.75) {
				$mr$
			};
			\node[draw] (sr) at (4,-1.75) {
				$sr$
			};
			
			\path[->]
			%attack 
			(argm) edge (sr)
			(argm) edge[bend right] (argl)
			(argl) edge (mr)
			%support 
			(cc) edge[dotted] (mr)
			(argl) edge[dotted] (sr)
			;
			
		\end{tikzpicture}
	\end{center}
	The BAF reflects the fact that $cc$ (in favor of climate change) can only be accepted when $mr$ (more rain) is also included in the extension. 
	As desired, $cc$ is acceptable; for instance, 
	$\{ cc, mr \}\in\com(D)$. 
	Correspondingly, 
	$\{ cc, mr, A_1,A_3 \}\in\com(\CF_D)$ so the acceptability of $cc$ is also found in $\CF_D$. 
\end{example}

\section{BAFs and Admissible Semantics}	
\label{sec:pbafs}
 
	\renewcommand{\PF}{\mathbb{F}}
	\newcommand{\allprem}{\mathcal{P}}
	\newcommand{\prem}{\ensuremath{\pi}}

	Our analysis in Section~\ref{sec:ABA vs BAF} reveals that we cannot capture admissible (and consequently preferred) ABA semantics by means of our \FT{instantiated} BAFs since there is no %tool 
 \FT{way} to guarantee that the accepted sets of arguments are assumption exhaustive% in the constructed graph $\CF_D$
 . 
	In this section, we will propose a slightly augmented version of BAFs which additionally stores this information. 
	This proposal is in line with recent developments in AF generalizations which capture certain features of instantiated graphs in addition to a purely abstract view  \cite{DvorakW20,Rapberger20,DBLP:journals/jair/RapbergerU23}. 
	In Section~\ref{sec:computational complexity} we will see that the computational price we have to pay for this is moderate. 
	
	\begin{definition}
		Let $\allprem$ be a set (of premises). 
		A \emph{premise-augmented BAF} (pBAF) $\PF$ is a tuple 
		$\PF = (\args, \attacker, \supporter, \prem)$ 
		where 
		$\CF= (\args, \attacker, \supporter)$ is a BAF and 
		$\prem:\args\to 2^\allprem$ 
		is the \emph{premise function};  
		%The BAF 
		%$\CF = (\args, \attacker, \supporter)$ 
		$\CF$ 
		is called the \emph{underlying} BAF of $\PF$. 
	\end{definition}
        We let $\prem(E) = \bigcup_{a\in E} \prem(a)$. 
	We sometimes abuse notation and write 
	$\PF = (\CF, \prem)$ 
	for the pBAF 
	$\PF = (\args, \attacker, \supporter, \prem)$ 
	with underlying BAF 
	$\CF= (\args, \attacker, \supporter)$. 
	The following properties are defined due to the underlying BAF $\CF$: 
	\begin{definition}
		Let $\PF = (\CF,\prem)$ be a pBAF. 
		A set $E\subseteq \args$ is \emph{conflict-free} resp. \emph{closed} whenever this is the case for $E$ in $\CF$; $E$ \emph{defends} $a\in A$ in $\PF$ iff this is the case in $\CF$. 
	\end{definition}
	The only novel concept we require is the notion of an exhaustive set of arguments. 
	\begin{definition}
		Let $\PF = (\CF,\prem)$ be a pBAF. 
		A set $E\subseteq \args$ is \emph{exhaustive} iff $\prem(a)\subseteq \prem(E)$ 
        implies $a\in E$. 
	\end{definition}
	Semantics for pBAFs are defined similar\FT{ly} as for BAFs, but with the important difference that we require all admissible-based extensions to be exhaustive.
	\begin{definition}
		For a pBAF $\PF = (\CF,\prem)$, a set 
		$E\in\cf(\CF)$ is  
		\begin{itemize} 
			\item \emph{admissible}, $E\in\adm(\PF)$, iff $E$ is exhaustive and  $E\in\adm(\CF)$; 
			\item \emph{preferred}, $E\in\prf(\PF)$, iff it is \FT{$\subseteq$-}maximal admissible; 
			\item \emph{complete}, $E\in\com(\PF)$, iff $E\in\adm(\PF)$ and $E = \Gamma(E)$; 
			\item \emph{grounded}, $E\in\grd(\PF)$, iff $E = \bigcap_{S\in\com(\PF)}S$;
			\item \emph{stable}, $E\in\stb(\PF)$, iff it is closed and $E_\CF^+ = A\setminus E$. 
		\end{itemize}
	\end{definition}
	\begin{example}
		Let us illustrate how pBAFs can help us fixing our issue illustrated in Example~\ref{ex:nont assumption exhaustive mismatch}. 
		Let us construct the same BAF, but assign to each argument the assumptions required to entail it as premises. 
		
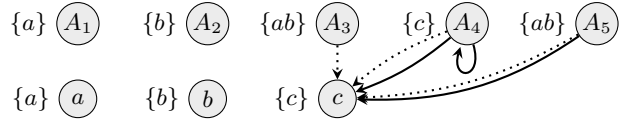
\begin{center}
			\begin{tikzpicture}[>=stealth,xscale=1.15, yscale=1]
				\small
				\path
				(0,0) node[arg,label=left:$\{ab\}$](a3){$A_3$}
				(-1.5,0) node[arg,label=left:$\{b\}$](a2){$A_2$}
				(-3,0) node[arg,label=left:$\{a\}$](a1){$A_1$}
				(-3,-1) node[arg,label=left:$\{a\}$](a){$a$}
				(-1.5,-1) node[arg,label=left:$\{b\}$](b){$b$}
				(0,-1) node[arg,label=left:$\{c\}$](c){$c$}
				(1.5,0) node[arg,label=left:$\{c\}$](a4){$A_4$}
				(3,0) node[arg,label=left:$\{ab\}$](a5){$A_5$}
				;
				\path [->,thick]
				(a4) edge[loop below] (a4)
				(a4) edge[bend left=10] (c)
				(a5) edge[bend left=20] (c)
				;
				\path [->,thick, dotted]
				(a4) edge[bend right=10] (c)
				(a5) edge[bend left=15] (c)
				(a3) edge[] (c)
				;
			\end{tikzpicture}
		\end{center}
		We have that, for instance, $E = \{a,A_1\}\in\adm(\PF)$. 
		Both arguments are unattacked with no out-going support arrow. Thus the only condition to verify is exhaustiveness. 
		This property is satisfied since no further argument $x$ satisfies $\prem(x)\subseteq \prem(E) = \{a\}$. 
		
		On the other hand, $E' = \{a,b,A_1,A_2\}$ is not admissible. Since $\prem(E') = \{a,b\}$, exhaustiveness would also require presence of $A_3$ and $A_5$ (which, in turn, would result in acceptance of $c$ which cannot be defended). 
	\end{example}
	Following the observations made in this example, 
	we define $\PF_D$ as follows: The underlying BAF $\CF$ is given as before and $\prem$ stores the assumptions required to entail an argument. 
	\begin{definition}
		For an ABAF  
		$D = (\mathcal{L},\mathcal{R},\mathcal{A},\contraryempty)$, 
		the %corresponding 
  \emph{\FT{instantiated} pBAF} 
		$\PF_D = (\args,\attacker,\supporter,\prem) = (\CF, \prem)$ 
		is %given as
        \begin{align*}
			\CF &= \CF_D && \forall x \in \args: \prem(x) = \asms(x). 
		\end{align*} 
	\end{definition}
	We can now capture any non-flat ABAF as follows. 
	\begin{restatable}{theorem}{thSemanticsCorrespondenceADM}
		\label{th:semantics correspondence 2}
		Let $D = (\mathcal{L},\mathcal{R},\mathcal{A},\contraryempty)$ be an ABAF 
		and $\PF_D = (\CF,\prem)$ the %corresponding 
  \FT{instantiated} pBAF. Then  
		\begin{itemize}
			\item if $E\in\sigma(\PF_D)$, then $\asms(E)\in\sigma(D)$; 
			\item if $S\in\sigma(D)$, then $\{ x\in A \mid \asms(x)\subseteq S \}\in\sigma(\PF_D)$ 
		\end{itemize}
		for any $\sigma\in\{\adm,\com,\prf,\grd,\stb\}$. 
	\end{restatable}

	\section{Computational Complexity} 
	\label{sec:computational complexity}
%	In order to emphasize the difference between our setting and instantiating flat ABAFs, we recall the complexity of flat ABAFs and AFs in Table~\ref{table:flat complexity}. 
%	
%	\begin{table}[bt] 
%		\small
%		\setlength\tabcolsep{5pt}
%		\def\arraystretch{1.2}
%		\centering
%		\begin{tabular}{|ll|c|c|c|c|c|}
%			\hline
%			&  & $\adm$ & $\grd$  & $\com$  & $\pref$ & $\stb$ \\ \hline
%			\multicolumn{1}{|l|}{\multirow{2}{*}{$\Ver_{\sigma}$}}                                               & ABA & in $\P$  & in $\P$   & in $\P$  & $\coNP$-c & in $\P$  \\
%			\multicolumn{1}{|l|}{}                                                                               & AF & in $\P$  & in $\P$   & in $\P$  & $\coNP$-c & in $\P$  \\ \hline
%			\multicolumn{1}{|l|}{\multirow{2}{*}{$\Cred_{\sigma}$}}                                              & ABA  & $\NP$-c & in $\P$  & $\NP$-c & $\NP$-c  & $\NP$-c   \\
%			\multicolumn{1}{|l|}{}                                                                               & AF  & $\NP$-c & in $\P$  & $\NP$-c & $\NP$-c  & $\NP$-c \\\hline
%			\multicolumn{1}{|l|}{\multirow{2}{*}{$\Skept_{\sigma}$}}                                             & ABA  & triv.\ & in $\P$  & in $\P$  & $\PiP{2}$-c  & $\DP$-c   \\
%			\multicolumn{1}{|l|}{}                                                                               & AF  & triv.\  & in $\P$  & in $\P$ & $\PiP{2}$-c & $\DP$-c \\\hline
%		\end{tabular}
%		
%		\caption{Computational Complexity: Flat ABA vs. AFs}
%		\label{table:flat complexity}
%	\end{table}

    We consider the usual decision problems (under semantics $\sigma$) in formal argumentation. 
    Let $\mathcal K$ be a knowledge base (\ie an ABAF, BAF, or pBAF), let $a$ be an assumption resp. argument, and let $E$ be a set of assumptions resp. arguments. 
\begin{itemize}
	\item Credulous acceptance $\Cred_\sigma$: Given $\mathcal K$ and some $a$, is it true that $a\in E$ for some $E\in\sigma(\mathcal K)$? 
	\item Skeptical acceptance $\Skept_\sigma$: Given $\mathcal K$ and some $a$, is it true that $a\in E$ for each $E\in\sigma(\mathcal K)$? 
	\item Verification $\Ver_\sigma$: Given $\mathcal K$ and a set $E$, is it true that $E\in\sigma(\mathcal K)$? 
\end{itemize}

	We start with the computational complexity of BAFs with our novel semantics. 
	The high level observation is that many tasks are close to reasoning in usual AFs. However, computing the grounded extension is much harder, inducing certain consequences (e.g. there is no shortcut for skeptical reasoning under complete semantics). 
	\begin{restatable}{theorem}{thBAFComplexity}
		\label{th:BAF complexity}
		For BAFs, the problem 
		\begin{itemize}
			\item 
			$\Ver_{\sigma}$ is 
			tractable for $\sigma\in\{\adm,\com,\stb\}$, 
			$\coNP$-complete for $\sigma = \prf$, and  
			$\DP$-complete for $\sigma = \grd$. 
			\item 
			$\Cred_{\sigma}$ is 
			$\NP$-complete for $\sigma\in\{\adm,\com,\prf,\stb\}$ and 
			$\DP$-complete for $\sigma = \grd$. 
			\item 
			$\Skept_{\sigma}$ is 
			trivial for $\sigma = \adm$, 
			$\DP$-complete for $\sigma \in \{\com,\grd,\stb\}$, and
			$\Pi^P_2$-complete for $\sigma = \prf$.  
			%			\item 
			%			$\Exists_{\sigma}$ is 
			%			trivial for $\sigma \in \{\adm,\prf\}$ and  
			%			$\NP$-complete for $\sigma \in \{\com,\grd,\stb\}$. 
		\end{itemize}
	\end{restatable}
	Surprisingly, the price we have to pay for also capturing admissible-based semantics is rather small. 
	The computational complexity of the pBAFs we construct is almost the same; the only difference is that skeptical acceptance w.r.t.\ admissible semantics is not trivial anymore, but now becomes $\coNP$-complete. 
%	First, however, we require the following observation: 
%	When constructing $\PF_D$, we impose intrinsic constraints on the $\prem$ function, captured in the following definition. 
%	\begin{definition}
%		\label{def:well-formed pbaf}
%		Let $\PF = (\args,\attacker,\supporter,\prem)$ be a pBAF. 
%		Then $\PF$ is called \emph{well-formed} iff for each $E\subseteq E' \subseteq \args$, it holds that: 
%		i) $\prem(E)\subseteq \prem(E')$;  
%		ii) if $\prem(E) = \emptyset$, then $E^- = \emptyset$; 
%		iii) if $E$ supports some $x$, then $\prem(E)\neq \emptyset$. 
%	\end{definition}
%	\begin{restatable}{proposition}{propConstructedPBAFWellFormed}
%		For any ABAF $D$, the corresponding pBAF $\PF_D$ is well-formed. 
%	\end{restatable}
	\begin{restatable}{proposition}{propSkepAdmPBAF}
		For pBAFs, $\Skept_{\adm}$ is $\coNP$-complete. 
	\end{restatable}
	From this observation the following main theorem follows as a corollary of the complexity results for BAFs. 
	\begin{restatable}{theorem}{thPBAFComplexity}
		\label{th:wf pBAF complexity}
		For pBAFs, the problem 
		\begin{itemize}
			\item 
			$\Ver_{\sigma}$ is 
			tractable for $\sigma\in\{\adm,\com,\stb\}$, 
			$\coNP$-complete for $\sigma = \prf$, and  
			$\DP$-complete for $\sigma = \grd$. 
			\item 
			$\Cred_{\sigma}$ is 
			$\NP$-complete for $\sigma\in\{\adm,\com,\prf,\stb\}$ and 
			$\DP$-complete for $\sigma = \grd$. 
			\item 
			$\Skept_{\sigma}$ is 
			$\coNP$-complete for $\sigma = \adm$, 
			$\DP$-complete for $\sigma \in \{\com,\grd,\stb\}$, and
			$\Pi^P_2$-complete for $\sigma = \prf$.  
			%			\item 
			%			$\Exists_{\sigma}$ is 
			%			trivial for $\sigma \in \{\adm,\prf\}$ and  
			%			$\NP$-complete for $\sigma \in \{\com,\grd,\stb\}$. 
		\end{itemize}
	\end{restatable}
Table~\ref{table:non-flat complexity} summarizes our complexity results for BAFs and pBAFs and compares them to the known computational complexity of non-flat ABA \cite{DBLP:journals/ai/CyrasHT21}.\footnote{In their paper, the standard definition of the empty intersection is used, but the results can be derived analogously.} 
\begin{table}[bt] 
	\small
	\setlength\tabcolsep{4pt}
	\def\arraystretch{1.2}
	\centering
	\begin{tabular}{|ll|c|c|c|c|c|}
		\hline
		&  & $\adm$ & $\grd$  & $\com$  & $\pref$ & $\stb$ \\ \hline
		\multicolumn{1}{|l|}{\multirow{3}{*}{$\Ver_{\sigma}$}}                                               & 
		ABA & $\coNP$-c  & $\DP_2$-c  & $\coNP$-c & $\PiP{2}$-c & in $\P$  \\
		\multicolumn{1}{|l|}{}                                                                               & 
		BAF &in $\P$  & $\DP$-c  & in $\P$  & $\coNP$-c  & in $\P$   \\ 
		\multicolumn{1}{|l|}{}                                                                               & 
		pBAF &in $\P$  & $\DP$-c  & in $\P$  & $\coNP$-c  & in $\P$   \\ \hline
		\multicolumn{1}{|l|}{\multirow{3}{*}{$\Cred_{\sigma}$}}                                              & ABA  & $\SigmaP{2}$-c & $\DP_2$-c  & $\SigmaP{2}$-c  & $\SigmaP{2}$-c  & $\NP$-c   \\
		\multicolumn{1}{|l|}{}                                                                               & BAF  & $\NP$-c  & $\DP$-c  & $\NP$-c & $\NP$-c  & $\NP$-c \\
		\multicolumn{1}{|l|}{}                                                                               & pBAF  & $\NP$-c  & $\DP$-c  & $\NP$-c & $\NP$-c  & $\NP$-c \\\hline
		\multicolumn{1}{|l|}{\multirow{3}{*}{$\Skept_{\sigma}$}}                                             & ABA  & $\PiP{2}$-c & $\DP_2$-c  & $\DP_2$-c  & $\PiP{3}$-c  & $\DP$-c   \\
		\multicolumn{1}{|l|}{}                                                                               & BAF  & triv.\  & $\DP$-c  & $\DP$-c & $\PiP{2}$-c & $\DP$-c \\
		\multicolumn{1}{|l|}{}                                                                               & pBAF  & $\coNP$-c  & $\DP$-c  & $\DP$-c & $\PiP{2}$-c & $\DP$-c \\\hline
		%			\multicolumn{1}{|l|}{\multirow{2}{*}{$\Exists_{\sigma}$}}                                            & ABA  & $\SigmaP{2}$-c & $\SigmaP{2}$-c & $\SigmaP{2}$-c & $\SigmaP{2}$-c & $\NP$-c   \\
		%			\multicolumn{1}{|l|}{}                                                                               & BAF  & triv.\  & $\NP$-c  & $\NP$-c & triv.\  & $\NP$-c \\\hline
	\end{tabular}
	
	\caption{Complexity: Non-flat ABA vs. (p)BAFs}
	\label{table:non-flat complexity}
\end{table}
We observe that for each reasoning problem we consider, pBAFs are one level below non-flat ABA in the polynomial hierarchy. 
Moreover, BAFs and pBAFs are comparable for most reasoning problems, with skeptical reasoning for $\adm$ semantics being the only exception. 
This is in line with our results that pBAFs are capable of capturing admissible reasoning in ABA (cf. Theorem~\ref{th:semantics correspondence 2}), whereas BAFs are not. 
Note that Theorem~\ref{th:semantics correspondence 2} does not contradict the differing results pertaining to computational in complexity in ABA vs.\ pBAFs since the instantiation procedure yields exponentially many arguments in general. 
%We discuss the potential implications in Section~\ref{sec:conc}. 

	\section{Discussion and Related Work}

	There 
	%are many 
	is a rich selection of 
	bipolar 
	argumentation approaches in the literature \cite{cayrol2021higher}. The most prominent
	ones are deductive \cite{boella2010support}, necessary \cite{nouioua2010bipolar},
	evidential \cite{oren2008semantics}
	and backing \cite{cohen2012backing} support.
	More recent work on classical BAFs looked at
	symmetry between attack and support \cite{potyka2020bipolar},
	argument attributes \cite{gonzalez2021labeled} and
	monotonicity \cite{gargouri2021notion}.
	
	Our notion of defense can be characterized using notions of extended attacks that occur in BAFs
	\cite{AmgoudCLL08,boella2010support}.
	There is a \emph{mediated attack from $a$ to $b$} if $a$ attacks %and 
 \FT{an} argument $c$ that is transitively supported by $b$ \cite{boella2010support}. Using our notion of closure from Definition  \ref{def_closure_baf}, this can be rewritten as $a$ attacks $cl(\{b\})$.
	Hence, Lemma \ref{le:defense closure single arg} states that $E$ defends $a$ iff for each attacker 
	$b$ of $a$, there is a direct or mediated attack from $E$ to~$b$.
	
	Another approach is due to \cite{AmgoudCLL08}. 
	%extend standard semantical notions in %(attack-only) 
	%Dung-AFs in another way to BAFs.
	Here a set of arguments $S$ defends~\cite{Dung95} an argument $a$ if $S$ attacks every attacker of $a$% (the standard Dung notion)
	. 
	Further,  $S$ attacks an argument $a$ iff there is a direct attack
	from $S$ to $a$ or $S$ transitively supports an argument $b$ that attacks $a$. 
 This amounts to requiring
 %This can also be expressed by just saying 
 that $cl(S)$ attacks $a$. A set of arguments $S$ is then called \emph{conflict free} %in their framework 
	if
	$S$ does not attack any argument in $S$ in the previous sense, that is, if $cl(S)$ does not directly attack any argument in $S$. 
	They call $S$ \emph{admissible} iff $S$ is conflict-free, closed under support  and defends all its elements.
	One important difference to our notion of admissible sets is the definition of defense. While \cite{AmgoudCLL08} allow defense via direct and supported attacks, we allow defense via direct or mediated attacks.
	For instance in 
	$\CF= ( \{a,b,c\}, \{ (a,c), (b,a) \}, \{ (b,c) \} )$, 
%	\begin{center}
%		\begin{tikzpicture}[>=stealth,xscale=1, yscale=1]
%			\small
%			\path
%			(-1.5,0) node[arg,label=left:$\CF:$](a){$a$}
%			(0,0) node[arg](c){$c$}
%			(1.5,0) node[arg](b){$b$}
%			;
%			\path [->,thick]
%			(b) edge[dotted] (c)
%			(a) edge (c)
%			(b) edge[bend left=20] (a)
%			;
%		\end{tikzpicture}
%	\end{center}
	$\{a\}$ is admissible w.r.t.\ our definition
	because it defends itself against $b$ via a mediated attack, is closed under support and conflict-free.
	However, it is not admissible w.r.t.\ the definition in \cite{AmgoudCLL08} because here, it does not defend itself against $b$.
	%w.r.t.\ their definition.%
	%Dispute trees have been used to reason about
	%and to explain other argumentation formalisms 
	%\cite{dung2006dialectic,vcyras2022dispute}.
	%A thorough overview of the complexity of
	%reasoning in ABA can be found in
	%\cite{vcyras2021computational}.
	
	The c-admissibility in \cite{cayrol2005acceptability} is close to ours, but their work uses \emph{supported} and \emph{indirect} defeat which is incompatible with our concepts. 
	
	\cite{CyrasFST17} also explored the relation %between ABAFs and BAFs 
	ABAFs-BAFs, but understands the BAFs under existing semantics in terms of a restricted form of non-flat-ABAFs, rather than
	%Instead, we aimed at understanding any general 
	ABAF as BAF, under new semantics, as we do. 
	\section{Conclusion}
	\label{sec:conc}
	%In this paper, 
	We translated non-flat ABAFs into BAFs. 
	To this end we proposed novel BAF semantics to capture complete-based semantics. 
    By means of a novel formalisms, called \emph{premise-augmented} BAFs, we also established a correspondence between admissible-based semantics of non-flat ABA and our BAFs. 
	We discussed basic properties of these semantics and %then formally 
	proved the correspondence to ABA.  
	We then investigated the computational complexity of BAFs and showed that\FT{,} compared to non-flat ABA, the typical reasoning problems are one level lower in the polynomial hierarchy. 
	
	This work opens several avenues for future work. It would be interesting to  extend our results to further ABA semantics like semi-stable and ideal semantics~\cite{CyrasFST2018}, as well as a set-stable semantics~\cite{bipolarABA}. Further, we only discussed basic properties of our (p)BAF semantics: 
	it would also be interesting to study a version of the fundamental lemma, existence of extensions, and the impact of restricting the graph class. 

    The lower computational complexity of (p)BAFs compared to non-flat ABAFs raises the question \FT{as} to which extent our research can contribute to the development of efficient instantiation-based ABA solvers. 
    As for the flat ABA instantiation by means of AFs, our constructed graphs are infinite in general. 
    However, in the case of AFs, only finitely many arguments suffice to determine the semantics, and techniques to reduce the size of the constructed AF even further are available \cite{DBLP:conf/kr/LehtonenR0W23}.

    As a final remark, most approaches for explaining reasoning in ABA construct the underlying AF and extract a\FT{n} argumentative explanation from it \cite{DBLP:conf/ijcai/Cyras0ABT21}. 
    Our research serves as the first step to enabling this machinery %to 
    \FT{for} non-flat ABAFs as well: %T
    \FT{t}he next goal would be the computation of intuitive explanations in (p)BAFs. 
	This line of research would contribute to applications where non-flat ABAFs give natural representations, as in Example~\ref{ex:motivating} as well as settings where agents share information (e.g. as in \cite{bipolarABA}) and may thus disagree on which information is factual and defeasible. 
	
\section*{Acknowledgements}
This research was partially funded by the European Re-
search Council (ERC) under the European Union’s Horizon 2020 research and innovation programme (grant agreement No. 101020934, ADIX), by J.P. Morgan and by the Royal Academy of Engineering under the Research Chairs and Senior Research Fellowships scheme, and 
by the Federal Ministry of Education and Research of Germany and by S\"achsische Staatsministerium für Wissenschaft, Kultur und Tourismus in the programme Center of Excellence for AI-research ``Center for Scalable Data Analytics and Artificial Intelligence Dresden/Leipzig'', project identification number: ScaDS.AI.
Any views or opinions
expressed herein are solely those of the authors.

	%\end{document}
	
	\clearpage	
	
	\appendix
	%\onecolumn
	\section*{Non-flat ABA is an Instance of Bipolar Argumentation -- Supplementary Material}
	\section{Proof Details of Section~\ref{sec:ABA vs BAF}}

        \thSemanticsCorrespondence*
        \begin{proof}
            This main result is a corollary of the subsecquent propositions. 
        \end{proof}
 
	We start with two auxiliary lemmata. 
	\begin{lemma}
		\label{le:reduce defense to single arg}
		Let $D = (\mathcal{L},\mathcal{R},\mathcal{A},\contraryempty)$ be an ABAF. 
		A set $S\subseteq \mathcal A$ of assumptions defends some $a\in \mathcal A$ 
		iff 
		for each tree-based argument $T\vdash p$ attacking $a$, $S$ attacks $\cl(T)$. 
	\end{lemma} 
	\begin{proof}
		($\Rightarrow$) 
		is clear since $\cl(T)$ is a closed set of assumptions attacking $a$. 
		
		($\Leftarrow$) 
		If $T$ is closed set of assumptions attacking $a$, then $\bar a\in\theory_D(T)$. 
		Take an argument $T'\vdash \bar a$ with $T'\subseteq T$. 
		By assumption $S$ attacks $\cl(T')$. 
		Since $\cl(T')\subseteq \cl(T)$, $S$ attacks $\cl(T)$ as well. 
	\end{proof}
	We can therefore reduce our attention to the closure of single arguments, instead of considering arbitrary closed sets of assumptions. 
	\begin{lemma}
		\label{le:closure of single arguments}
		Let $D = (\mathcal{L},\mathcal{R},\mathcal{A},\contraryempty)$ be an ABAF 
		and $\CF_D = (A,\attacker,\supporter)$ the corresponding BAF. 
		Let $x = S\vdash p\in A$ be a tree-based argument. 
		Then $\asms(\cl(x)) = \cl(S)$. 
	\end{lemma}
	\begin{proof}
		($\subseteq$) 
		Let $a\in \asms(\cl(x))$. 
		If $a\in S$, then clearly $a\in \cl(S)$. 
		So suppose $a\notin S$. 
		First we observe that if $y\in\cl(x)$, then $(x,y)\in\supporter$, \ie it suffices to consider one step of supports. 
		To see this assume $(x,y)\in\supporter$ and $(y,z)\in\supporter$. 
		By definition, $y$ is of the form $\{b\}\vdash b$ for some assumption $b$ and analogously, $z = \{c\}\vdash c$. 
		From $(x,y)\in\supporter$ we deduce $b\in\cl(S)$ by definition of $\supporter$. 
		Now $(y,z)\in\supporter$ yields $c\in\cl(b)$. 
		But then $c\in\cl(S)$ as well. 
		Hence again by definition, $(x,z)\in\supporter$. 
		
		Thus, if $a\notin S$ but $a\in\cl(x)$, then it must be the case that $(x,\{a\}\vdash a)\in\supporter$. 
		By definition of $\supporter$ this means $a\in\cl(S)$. 
		
		$(\supseteq)$ 
		Suppose $a\in \cl(S)$. If $a\in S$, then $a\in\asms(x)$ and we are done. 
		Otherwise $(x,\{a\}\vdash a)\in\supporter$ by construction. 
		Thus $\{a\}\vdash a\in\cl(x)$ and hence $a\in\asms(\cl(x))$. 
	\end{proof}
	
	\propABAToAFAdm*
	\begin{proof}
		(conflict-free) 
		Suppose there is some argument $x\in E$ attacking $E$. 
		Let $\conc(x) = \bar a$; then $E$ must contain some argument $y$ with $a\in\asms(y)$. 
		Then $a\in S$. 
		Moreover, $x\in E$ implies $S\vdash \conc(x) = \bar a$. 
		Thus, $S$ is not conflict-free, a contradiction. 
		
		(closed) 
		Let $x = \{a\}\vdash a$ be an argument in $\CF_D$ s.t. $(y,x)\in\supporter$ for some $y\in E$. 
		For $y$ of the form $T\vdash p$ we have by definition $a\in\cl(T)$. 
		By choice of $E$, $T\subseteq S$ and hence $\cl(T)\subseteq \cl(S)$. Since $S$ is closed we deduce $a\in S$.  
		Again by choice of $E$, $x\in E$. 
		
		(defense) 
		Suppose $E'$ is a closed set of arguments attacking $E$. 
		Let us first assume $E'$ is of the form $E' = \cl(x)$ for some $x\in A$. 
		Let $T = \asms(x)$. Then $\cl(T) = \asms(E')$ by Lemma~\ref{le:closure of single arguments}. 
		By admissibility of $S$, $S\vdash \bar a$ for some $a\in \cl(T)$, \ie some $a\in\asms(E')$. 
		By definition, there is some tree-based argument $y = S'\vdash \bar a$ with $S'\subseteq S$. 
		By choice of $E$ we have $y\in E$ due to $\asms(y) = S'\subseteq S$. 
		Hence $E\to E'$ in $\CF_D$, \ie $E$ defends itself against $E'$ as desired. 
		Now for the general case suppose $E'$ is an arbitrary closed set of arguments attacking $E$. 
		Observe that for each $x\in E'$, $\cl(x)\subseteq \cl(E')$. 
		Hence take any $x\in E'$ attacking $E$. By the above reasoning, $E\to \cl(x)$ and thus $E\to E'$. 
	\end{proof}
	\propABAToAFCom*
	\begin{proof}
		Since admissibility is already established, we have left to show: 
		
		(fixed-point) 
		Suppose $E$ defends $x\in A$. 
		We have to show that $x\in E$. 	
		Let $a\in\asms(x)$. 
		We show that $a$ is defended by $S$ in $D$. 
		We make use of Lemma~\ref{le:reduce defense to single arg} and consider some tree-based argument $T\vdash \bar a$, \ie it attacks $a$. 
		
		Now $T\vdash \bar a$ attacks $x$ in $\CF_D$. 
		Since $E$ defends $x$, there is some $S'\subseteq S = \asms(E)$ with $S'\vdash \bar b$ for some $b\in \cl(T)$; in particular, this argument $S'\vdash \bar b$ is contained in $E$.  
		Thus, $\bar b\in \theory_D(S)$ and therefore $S$ counter-attacks $\cl(T)$ in $D$. 
		
		As $T\vdash \bar a$ was an arbitrary attacker of $a$, Lemma~\ref{le:reduce defense to single arg} ensures that $S$ defends $a$. 
		Completeness of $S$ thus implies $a\in S$. 
		Since $a$ was an arbitrary assumption in $\asms(x)$, we deduce $\asms(x)\subseteq S$. 
		By construction of $E$, $x\in E$ as desired.  
	\end{proof}
	\propABAToAFStb*
	\begin{proof}
		We know already that $E$ is conflict-free and closed. 
		
		($\stb$) 
		Suppose $x\in A\setminus E$. 
		Let $x$ be of the form $T\vdash p$. 
		By construction of $E$, $T\setminus S\neq \emptyset$; consider some $a\in T\setminus S$. 
		Since $S$ is stable, $S$ attacks $a$, \ie there is some tree-based argument $S'\vdash \bar a$ with $S'\subseteq S$. 
		We have $S'\vdash \bar a\in E$ and hence $E$ attacks $x$. 
		%	
		%	($\sstb$) 
		%	Again let $x\in A\setminus E$ be of the form $T\vdash p$. 
		%	As above, $T\setminus S\neq \emptyset$ and we consider some $a\in T\setminus S$. 
		%	Since $S$ is set-stable, $S$ attacks $\cl(a)$, 
		%	Hence there is some $b\in\cl(a)$ and $S'\vdash \contrary{b}$ a tree-based argument where $S'\subseteq S$.  
		%	In particular, $S'\vdash \contrary{b}\in E$. 
		%	By definition of the support relation $\supporter$, $x = T\vdash p$ supports $\{b\}\vdash b$ due to $b\in\cl(T)$ (because $a\in T$). 
		%	Hence $E$ attacks $\cl(x)$. 
	\end{proof}
	
	For the other direction, we again start with an auxiliary observation. 
	
	\begin{lemma}
		\label{le:closed contains asms}
		Let $D = (\mathcal{L},\mathcal{R},\mathcal{A},\contraryempty)$ be an ABAF 
		and $\CF_D = (A,\attacker,\supporter)$ the corresponding BAF. 
		If $E$ is an assumption exhaustive and closed set of arguments in $\CF_D$ with $a\in\cl(\asms(E))$, then $a\in\asms(E)$. 
		In particular, any argument of the form	$T\vdash \bar a$ attacks $E$. 
	\end{lemma}
	\begin{proof}
		Since $a\in\cl(\asms(E))$, there is some argument $T\vdash a$ with $T\subseteq\asms(E)$. 
		Since $E$ is assumption exhaustive, $T\vdash a\in E$. 
		Moreover, $a\in\cl(T)$ and hence $(T\vdash a,\{a\}\vdash a)\in \supporter$. 
		Since $E$ is closed, $\{a\}\vdash a\in E$. 
		Hence $a\in\asms(E)$. 
	\end{proof}
	
	\propAssumptionExhaustiveAdm*
	\begin{proof}
		(conflict-free) 
		Suppose $S\vdash \bar a$ with $a\in S$. 
		Then, there is a tree-based argument $x = S'\vdash \bar a$ with $S'\subseteq S$ in $\CF_D$. 
		By Lemma~\ref{le:closed contains asms}, $x$ attacks $E$ in $\CF_D$. 
		Since $S'\subseteq S = \asms(E)$, we infer $x\in E$ since $E$ is assumption exhaustive. 
		Hence $E$ is not conflict-free, a contradiction. 
		
		(closed) 
		Suppose $a\in \cl(S)$. 
		Then there is some $S'\subseteq S$ with $S'\vdash a$. 
		Since $E$ is assumption exhaustive, $E$ contains each argument $x\in A$ with $\asms(x)\subseteq S$; in particular we obtain $S'\vdash a\in E$. 
		Since $(S'\vdash a,\{a\}\vdash a)\in\supporter$ and $E$ is closed, $\{a\}\vdash a\in E$ follows. 
		Hence $a\in\asms(E) = S$. 
		
		(defense)
		Let $S'$ be a closed set of assumptions attacking $S$. 
		Applying Lemma~\ref{le:reduce defense to single arg} we suppose $S' = \cl(T)$ for some tree-based argument $x = T\vdash \bar a$ with $a\in S$. 
		Since $a\in\asms(E)$, and $E$ is assumption exhaustive, $x$ attacks $E$ in $\CF_D$. 
		By admissibility of $E$, there is some $S'\subseteq S$ and a tree-based argument $S'\vdash \bar b\in E$ s.t.\ $b\in \cl(T)$. 
		Since $S'\subseteq S$, we have $\bar b\in\theory_D(S)$ as well and hence, $S$ counter-attacks $S'$ in $D$. 	
	\end{proof}	
	
	\leCoContainsAllAsms*
	\begin{proof}
		Let $Y\subseteq A$ be a closed set of arguments attacking $x = S \vdash p$ with $S\subseteq \asms(E)$. 
		Then some $y \in Y$ is of the form $T\vdash \bar a$ with $a\in S$. 
		Since $a\in\asms(E)$, $y$ attacks $E$. % by Lemma~\ref{le:closed contains asms}. 
		Hence $Y$ attacks $E$ and thus $E$ attacks $Y$ since $E$ defends itself.  
		Since $Y$ was an arbitrary closed attacker of $x$, $E$ defends $x$. 
		Since $E$ is complete, $x\in E$. 
	\end{proof}
	
	\propBAFToABACom*
	\begin{proof}	
		Since $E$ is assumption exhaustive, we have access to Proposition~\ref{prop:assumption exhaustive adm}. In addition, we have left to show: 
		
		(fixed-point) 
		Suppose $S$ defends $a\in\mathcal A$ in $D$. We have to show $a\in S$. 
		For this, it suffices to show that $E$ defends $\{a\}\vdash a$, because this implies $\{a\}\vdash a\in E$ (completeness) and hence $a\in\asms(E) = S$. 
		Thus, let $E'$ be a closed set of arguments in $\CF_D$ attacking $\{a\}\vdash a$. 
		
		We first assume $E' = \cl(x)$ for some $x = T\vdash \bar a$. 
		By assumption, $S$ defends $a$ and thus, $\bar b\in\theory(S)$ for some $b\in \cl(T)$. 
		Hence there is some tree-based argument $S'\vdash \bar b$ with $S'\subseteq S$. 
		Since $S'\subseteq S = \asms(E)$, we have that $S'\vdash \bar b\in E$ by Lemma~\ref{le:co contains all asms}. 
		It follows that $E$ counter-attacks $\cl(x)$ (recall Lemma~\ref{le:closure of single arguments}). 
		In case $E'$ is an arbitrary closed set of attackers, take one argument $x\in E'$ attacking $E$ and reason analogously. 
		
		Since $E'$ was an arbitrary closed attacker of $\{a\}\vdash a$ and since $E$ is complete, $\{a\}\vdash a\in E$ as desired. 
		We deduce $a\in S$ and again since $a$ was an arbitrary assumption defended by $S$, we have that $S$ contains all defended assumptions in $D$. 
	\end{proof}
	
	\propBAFToABAStb*
	\begin{proof}
		We know already that $S$ is conflict-free and closed due to Proposition~\ref{prop:assumption exhaustive adm}.
		%($\stb$)
		Now let $a\in \mathcal A$ with $a\notin S$. 
		Then $a\notin \asms(E)$ and thus $E$ attacks $\{a\}\vdash a$ since $E$ is stable in $\CF_D$. 
		We deduce $\bar a\in\theory_D(\asms(E))$ and hence $S$ attacks $a$ in $D$.
	\end{proof}
	%    \propBAFToABASetstb*
	%	\begin{proof}
		%		We know already that $S$ is conflict-free and closed due to Proposition~\ref{prop:assumption exhaustive adm}.	
		%		Let $a\in \mathcal A$ with $a\notin S$. 
		%		Then $a\notin \asms(E)$ and thus $E$ attacks 
		%		$\cl( \{a\}\vdash a )$, say $\{b\}\vdash b$. 
		%		Then $\bar b\in\theory_D(\asms(E))$ and hence $S$ attacks $\cl(a)$ in $D$. 
		%	\end{proof}
	
	%\propExhaustiveBAFSpecialProperties*
	%\begin{proof}
	%	For defense, observe that 
	%	$E$ attacks $\cl(x)$ in $\CF_D$ 
	%	iff 
	%	$E$ attacks $\cl(x)$ in $\CF_D^{e}$ 
	%	since the additionally supported arguments do not contribute new assumptions to $\cl(x)$ in $\CF_D$ compared to $\CF_D^{e}$. 
	%	The other statements are immediate by construction of $\CF_D^{e}$.
	%\end{proof}

	\section{Proof Details of Section~\ref{sec:pbafs}}
	
	\thSemanticsCorrespondenceADM*
	\begin{proof}
		We have left to show the transfer of exhaustiveness, since then Proposition~\ref{prop:assumption exhaustive adm} is applicable establishing the connection for $\adm$ and thus also $\prf$. 
        For complete-based semantics, the result follows from the previos main theorem for BAFs. 
        The fact that exhaustiveness is preserved is almost immediate: 
		
		(exhaustive) 
		If $E\in\sigma(\PF_D)$, then $E$ is exhaustive, \ie $\prem(a)\subseteq \prem(E)$ implies $a\in E$. 
		By construction, $E$ is thus assumption exhaustive. 
		
		If $S\in\sigma(D)$, then $\{ x\in A \mid \asms(x)\subseteq S \}\in\sigma(\PF_D)$ is assumption exhaustive and thus $\prem(a)\subseteq \prem(E)$ implies $a\in E$ by definition of $\prem$. 
		
		Consequently, Proposition~\ref{prop:assumption exhaustive adm} is applicable from which the claim follows. 
	\end{proof}	
	
	\section{Proof Details of Section~\ref{sec:computational complexity}}

	\thBAFComplexity*
	
	Before heading into our membership and hardness results, let us give the constructions we require. 
	The following adaptation of the standard construction will construct a BAF which has at least one complete extension iff $\Phi$ is satisfiable. 
	It makes use of the same technique we already applied for our introductory BAF examples: An unattacked argument $\top$ defends $\varphi$, and hence each complete extension must defend the latter.
	
	\begin{construction}
		\label{constr:co non-empty reduction}
		Let $\Phi$ a propositional formula over atoms $X = \{x_1,\ldots,x_n\}$ in 3-CNF which we identify with a set $C$ of clauses. 
		Let $\CF_\Phi=(A,\attacker,\supporter)$ where 
		\begin{align*}
			A =& \{x_i, \bar{x}_i \mid 1 \leq i \leq n\} \cup \{c  \mid c \in C\} \cup \{ \top,\varphi \} \\
			\attacker =& \{(x_i, \bar{x}_i), (\bar{x}_i, x_i)\mid 1 \leq i \leq n\} \cup\\   
			& \{(x_i, c) \mid x_i \in c \in C\} \cup \{(\bar{x}_i, c) \mid \neg x_i \in c \in C\} \cup \\
			& \{(c, \varphi) \mid c \in C\}\\
			\supporter =& \{(\top,\varphi)\}
		\end{align*}
		An example of this construction can be found in Figure~\ref{fig:example co non-empty reduction}. 
		
		\begin{figure}
			\begin{center}
				\begin{tikzpicture}[>=stealth,scale=0.8]
					\footnotesize
					\tikzstyle{args}=[circle,draw=black, minimum size=5mm, inner sep=0pt]
					\path 	node[args](t){$\varphi$}
					++(-2,0) node[args](top){$\top$}
					++(-1,-1.3) node[args](c1){$c_1$}
					++(3,0) node[args](c2){$c_2$}
					++(3,0) node[args](c3){$c_3$};
					\path 	(-5,-3.2)  node[args](x1){$x_1$}
					++(2.4,0) node[args](nx1){$\bar x_1$}
					++(1.3,0) node[args](x2){$x_2$}
					++(2.4,0) node[args](nx2){$\bar x_2$}
					++(1.3,0) node[args](x3){$x_3$}
					++(2.4,0) node[args](nx3){$\bar x_3$};
					\path [left,->, thick]	(c1) edge (t)
					(c2) edge (t)
					(c3) edge (t)
					;
					\path [left,->, thick]
					(x1) edge (c1)
					(x2) edge (c1)
					(nx2) edge (c2)
					(x3) edge (c2)
					(nx3) edge (c3)
					(nx1) edge (c3)
					;            
					\path [left,->, thick, bend left]
					(x1) edge (nx1)
					(nx1) edge (x1)
					(x2) edge (nx2)
					(nx2) edge (x2)
					(x3) edge (nx3)
					(nx3) edge (x3)
					;
					\path [->, thick, dotted]
					(top) edge (t)
					;
				\end{tikzpicture}
			\end{center}
			\caption{Construction~\ref{constr:co non-empty reduction} applied to the formula $ \Phi = \{x_1,x_2\}, \{\neg x_1, x_3\}, \{\neg x_1, \neg x_3\} $}
			\label{fig:example co non-empty reduction}
		\end{figure}
	\end{construction}
	\begin{lemma}
		\label{le:co non-empty reduction}
		Let $\Phi$ be a propositional formula in 3-CNF. 
		The BAF $\CF_\Phi$ as given in Construction~\ref{constr:co non-empty reduction} satisfies the following properties: 
		\begin{itemize}
			\item For each $E\in\com(\CF_\Phi)$, $\varphi\in E$ and $\top\in E$.  
			\item It holds that $\com(\CF_\Phi)\neq\emptyset$ iff $\Phi$ is satisfiable.  
		\end{itemize}
	\end{lemma}
	\begin{proof}
		Applying the usual reasoning for the standard construction, the second statement is a consequence of the first one. 
		To see the first statement, note that $E\in\com(\CF_\Phi)$ implies $\top\in E$. 
		Since $E$ must be closed, $\varphi\in E$. 
		
		In order for $\varphi$ to be defended, $E$ must contain $x_i$ resp. $\bar x_i$ arguments corresponding to a satisfying assignment of $\Phi$. 
		Hence $\varphi\in E\in\com(\CF_\Phi)$ is possible iff $\Phi$ is satisfiable.  
	\end{proof}
	Hence in the BAF constructed in Construction~\ref{constr:co non-empty reduction}, $\top$ and $\varphi$ are in the grounded extension and none of the $C_i$ are. 
	However, it might happen that some $x_i$ is in the intersection of all complete extensions because it might be true in every satisfying assignment. 
	In this case, $\grd(\CF_\Psi)$ does not only contain $\top$ and $\varphi$, but also all affected $x_i$ resp. $\bar x_i$ arguments. 
	In order to prevent this, we introduces copies $x'_i$ and $\bar x'_i$ of all variables. 
	They interact with the remaining AF analogously to $x_i$ resp. $\bar x_i$ and all four arguments mutually attack each other for each $i$. 
	This yields:% 
	\begin{construction}
		\label{constr:co non-empty reduction 2}
		Let $\Phi$ a propositional formula over atoms $X = \{x_1,\ldots,x_n\}$ in 3-CNF which we identify with a set $C$ of clauses. 
		Let $\CG_\Phi=(A,\attacker,\supporter)$ where 
		\begin{align*}
			A =& \{x_i, x'_i, \bar{x}_i, \bar{x}'_i\mid 1 \leq i \leq n\} \cup \{c  \mid c \in C\} \cup \{ \top,\varphi \} \\
			\attacker =& \{(x_i, \bar{x}_i), (x_i,x'_i), (x_i,\bar x'_i) \mid 1 \leq i \leq n\} \cup\\   
			& \{(\bar{x}_i,x_i), (\bar{x}_i,x'_i), (\bar{x}_i,\bar x'_i) \mid 1 \leq i \leq n\} \cup\\   
			& \{(x'_i, x_i), (x'_i,\bar x_i), (x'_i,\bar x'_i) \mid 1 \leq i \leq n\} \cup\\   
			& \{(\bar{x}'_i,x_i), (\bar{x}'_i,x'_i), (\bar{x}'_i,\bar x_i) \mid 1 \leq i \leq n\} \cup\\   
			& \{(x_i, c), (x'_i, c) \mid x_i \in c \in C\} \cup \\
			& \{(\bar{x}_i, c),(\bar{x}'_i, c) \mid \neg x_i \in c \in C\} \cup \\
			& \{(c, \varphi) \mid c \in C\}\\
			\supporter =& \{(\top,\varphi)\}
		\end{align*}
	\end{construction}
	Considering how Construction~\ref{constr:co non-empty reduction 2} is obtained from Construction~\ref{constr:co non-empty reduction} the following can be inferred. 
	\begin{lemma}
		\label{le:co non-empty reduction 2}
		Let $\Phi$ be a propositional formula in 3-CNF. 
		The BAF $\CG_\Phi$ as given in Construction~\ref{constr:co non-empty reduction 2} satisfies the following properties: 
		\begin{itemize}
			\item If $\Phi$ is satisfiable, then $\grd(\CG_\Phi) = \{ \top, \varphi \}$. 
			\item If $\Phi$ is not satisfiable, then $\grd(\CG_\Phi) = \emptyset$. 
		\end{itemize}
	\end{lemma}
	Next we consider some construction which will help us to show that skeptical reasoning with $\com$ is $\coNP$-hard for BAFs. 
	For this, we make use of a gadget which ensures that each complete extension corresponds to some satisfying assignment (and not to just a partial assignment) to the $X$ variables. 
	
	\begin{construction}
		\label{constr:co and unsat forumla}
		Let $\Psi$ a propositional formula over atoms $X = \{x_1,\ldots,x_n\}$ in 3-CNF which we identify with a set $C$ of clauses. 
		Let $\HH_\Psi=(A,\attacker,\supporter)$ where 
		\begin{align*}
			A =& \{x_i, \bar{x}_i, \top_i, \bot_i, d_i \mid 1 \leq i \leq n\} \cup \{c  \mid c \in C\} \cup \{ \top,\varphi \} \\
			\attacker =& \{(x_i, \bar{x}_i), (\bar{x}_i, x_i)\mid 1 \leq i \leq n\} \cup\\   
			& \{(x_i, c) \mid x_i \in c \in C\} \cup \{(\bar{x}_i, c) \mid \neg x_i \in c \in C\} \cup \\
			& \{ (x_i,\bot_i), (\bar x_i,\bot_i), (\bot_i,\bot_i), (\bot_i,d_i) \mid 1\leq i \leq n  \}\cup \\
			& \{(c, \varphi) \mid c \in C\}\\
			\supporter =& \{ (\top_i,d_i) \mid 1\leq i \leq n  \}
		\end{align*}
		An example of this construction can be found in Figure~\ref{fig:example co and unsat forumla}. 
		
		\begin{figure}
			\begin{center}
				\begin{tikzpicture}[>=stealth,scale=0.8]
					\footnotesize
					\tikzstyle{args}=[circle,draw=black, minimum size=5mm, inner sep=0pt]
					\path 	node[args](t){$\psi$}
					++(-2,0) node[args](top){$\bar \psi$}
					++(-1,-1.3) node[args](c1){$c_1$}
					++(3,0) node[args](c2){$c_2$}
					++(3,0) node[args](c3){$c_3$};
					\path 	(-5,-3.2)  node[args](x1){$x_1$}
					++(2.4,0) node[args](nx1){$\bar x_1$}
					++(1.3,0) node[args](x2){$x_2$}
					++(2.4,0) node[args](nx2){$\bar x_2$}
					++(1.3,0) node[args](x3){$x_3$}
					++(2.4,0) node[args](nx3){$\bar x_3$};
					\path 	(-5,-5)  node[args](top1){$\top_1$}
					++(1.2,0) node[args](d1){$d_1$}
					++(1.2,0) node[args](bot1){$\bot_{1}$}
					++(1.3,0) node[args](top2){$\top_2$}
					++(1.2,0) node[args](d2){$d_2$}
					++(1.2,0) node[args](bot2){$\bot_{2}$}
					++(1.3,0) node[args](top3){$\top_3$}
					++(1.2,0) node[args](d3){$d_3$}
					++(1.2,0) node[args](bot3){$\bot_{3}$};
					\path [left,->, thick]	(c1) edge (t)
					(c2) edge (t)
					(c3) edge (t)
					;
					\path [left,->, thick]
					(x1) edge (c1)
					(x2) edge (c1)
					(nx2) edge (c2)
					(x3) edge (c2)
					(nx3) edge (c3)
					(nx1) edge (c3)
					;            
					\path [left,->, thick, bend left]
					(x1) edge (nx1)
					(nx1) edge (x1)
					(x2) edge (nx2)
					(nx2) edge (x2)
					(x3) edge (nx3)
					(nx3) edge (x3)
					;
					\path [->, thick, loop below]
					(bot1) edge (bot1)
					(bot2) edge (bot2)
					(bot3) edge (bot3)
					;
					\path [->, thick]
					(x1) edge (bot1)
					(x2) edge (bot2)
					(x3) edge (bot3)
					(nx1) edge (bot1)
					(nx2) edge (bot2)
					(nx3) edge (bot3)
					;
					\path [->, thick]
					(t) edge (top)
					(bot1) edge (d1)
					(bot2) edge (d2)
					(bot3) edge (d3)
					;
					\path [->, thick, dotted]
					(top1) edge (d1)
					(top2) edge (d2)
					(top3) edge (d3)
					;
				\end{tikzpicture}
			\end{center}
			\caption{Construction~\ref{constr:co and unsat forumla} applied to the formula $ \Psi = \{x_1,x_2\}, \{\neg x_1, x_3\}, \{\neg x_1, \neg x_3\} $}
			\label{fig:example co and unsat forumla}
		\end{figure}
	\end{construction}
	\begin{lemma}
		\label{le:co and unsat forumla}
		Let $\Psi$ be a propositional formula in 3-CNF. 
		The BAF $\HH_\Psi$ as given in Construction~\ref{constr:co and unsat forumla} satisfies the following properties: 
		\begin{itemize}
			\item For each $1\leq i\leq n$ and each $E\in\com(\HH_\Psi)$ we have 
			\begin{itemize}
				\item $\top_i,d_i\in E$, 
				\item either $x_i\in E$ or $\bar x_i \in E$. 
			\end{itemize}
			\item The formula $\Psi$ is unsatisfiable iff $\bar \psi \in E$ for each $E\in\com(\HH_\Psi)$. 
			\item Let $G\in\grd(\HH_\Psi)$ be the grounded extension. The formula $\Psi$ is unsatisfiable iff $G = \{ \top_i,d_i\mid 1\leq i\leq n\}\cup \{\bar\psi\}$
		\end{itemize}
	\end{lemma}
	\begin{proof}
		The first statement follows from $\top_i\in E$. Hence $d_i$ needs to be defended and thus, $x_i\in E$ or $\bar x_i\in E$ (second statement). 
		
		Thus, each $E\in\com(H_\Psi)$ represents an assignment to the $X$-variables which implies the third statement (usual logic of the standard construction). 
	\end{proof}
	Equipped with these constructions, we are ready to prove Theorem~\ref{th:BAF complexity}. 
	\begin{proposition}[Verification] 
		\label{prop:com ver}
		The problem $\Ver_{\sigma}$ is 
		tractable for $\sigma\in\{\adm,\com,\stb\}$, 
		$\coNP$-complete for $\sigma = \prf$, and  
		$\DP$-complete for $\sigma = \grd$. 
	\end{proposition}
	\begin{proof}
		Let $\CF$ be a BAF. 
		
		($\adm$) 
		Given $E\subseteq A$ it is clear that checking $E\cap E^+ = \emptyset$ and $E = \cl(E)$ can be done in polynomial time. 
		For defense we use Lemma~\ref{le:defense closure single arg}: 
		Let $a$ be an attacker of $E$. We can compute $\cl(\{a\})$ in polynomial time and thus checking whether $E\to \cl(\{a\})$ holds is also tractable. 
		
		($\com$) 
		For completeness, we iterate additionally over each $a\in A\setminus E$ and check whether $a$ is defended by $E$. This is tractable as well. 
		
		($\com$) 
		We check $E\in\cf(\CF)$ and $E = \cl(E)$; then we compute the range and verify that each $a\in A\setminus E$ is attacked. 
		
		($\prf$) 
		Membership follows from tractability for the case $\adm$: Simply verify $E\in \adm(\CF)$ and then iterate over each superset of $E$ in order to verify maximality. 
		Hardness follows from hardness in the AF case. 
		
		($\grd$) 
		
		(membership)
		To verify that $E$ is the intersection of all complete extensions we proceed as follows: 
		\begin{itemize}
			\item Show that $E$ does not contain too many arguments: For each $S\subseteq A$, verify that $S\in\com(\CF)$ implies $E\subseteq A$ (universal quantifier in $\DP$); 
			\item Show that $E$ does not contain too few arguments: For the linearly many $a \in A\setminus E$, guess a set $E_a$ with $a\notin E_a$ and verify that $E_a\in\com(\CF)$ (existential quantifier in $\DP$).  
		\end{itemize}
		
		(hardness)
		Given an instance $(\Phi,\Psi)$ of the $\DP$-complete problem SAT-UNSAT we apply Constructions~\ref{constr:co non-empty reduction 2}~and~\ref{constr:co and unsat forumla} to obtain the combined AF $\CF = \CG_\Phi \cup \HH_\Psi$. 
		By Lemmata~\ref{le:co non-empty reduction 2}~and~\ref{le:co and unsat forumla} 
		$$G = \{\top,\varphi\} \cup \{ \top_i,d_i\mid 1\leq i\leq n\}\cup \{\bar\psi\}$$
		is the grounded extension of $\CF$
		iff 
		$\Phi$ is satisfiable and 
		$\Psi$ is unsatisfiable.
	\end{proof}
	
	\begin{proposition}[Credulous Reasoning] 
		\label{prop:com cred}
		The problem $\Cred_{\sigma}$ is 
		$\NP$-complete for $\sigma\in\{\adm,\com,\prf,\stb\}$ and 
		$\DP$-complete for $\sigma = \grd$. 
	\end{proposition}
	\begin{proof}
		Let $\CF$ be the given BAF. 
		For $\sigma\in\{\adm,\com,\prf,\stb\}$ a simple guess and check procedure suffices. 
		Note that credulous reasoning for preferred semantics coincides with credulous reasoning for admissible semantics; therefore we can apply Proposition~\ref{prop:com ver}. 
		Hardness follows from the AF case. 
		
		($\grd$)
		
		(membership) 
		To verify that $a\in A$ is in the intersection of all complete extensions we proceed as follows: 
		\begin{itemize}
			\item Show that $\grd(\CF)\neq\emptyset$ by guessing and verifying any $E\in\com(\CF)$ (existential quantifier in $\DP$).  
			\item Show that $a\in E$ for each $E\in\com(\CF)$ by iterating over each subset $S$ of $A$ and checking that $S$ is not complete or contains $a$. 
		\end{itemize}
		(hardness)
		Given an instance $(\Phi,\Psi)$ of the $\DP$-complete problem SAT-UNSAT we apply Constructions~\ref{constr:co non-empty reduction 2}~and~\ref{constr:co and unsat forumla} to obtain the combined AF $\CF = \CG_\Phi \cup \HH_\Psi$. 
		By Lemmata~\ref{le:co non-empty reduction 2}~and~\ref{le:co and unsat forumla} 
		$\bar\psi$ 
		is credulously accepted (\ie in the intersection of all complete extensions) 
		iff 
		$\Phi$ is satisfiable and 
		$\Psi$ is unsatisfiable.
	\end{proof}

	\begin{proposition}[Skeptical Reasoning] 
		\label{prop:com skep}
		The problem $\Skept_{\sigma}$ is 
		trivial for $\sigma = \adm$, 
		$\DP$-complete for $\sigma \in \{\com,\grd,\stb\}$, and
		$\Pi^P_2$-complete for $\sigma = \prf$.  
	\end{proposition}
	\begin{proof}
		Let $\CF$ be the given BAF. 
		
		For $\adm$ the answer is negative as usual. 
		
		For $\sigma \in \{\com,\stb,\grd\}$ membership is by guessing a witness for $\sigma(\CF)\neq\emptyset$ (existential quantifier) and then showing that each $E\in\sigma(\CF)$ contains the query argument (universal quantifier). 
		Hardness follows from again applying Constructions~\ref{constr:co non-empty reduction 2}~and~\ref{constr:co and unsat forumla} to obtain the combined AF $\CF = \CG_\Phi \cup \HH_\Psi$. 
		For $\grd$ we argued already in the proof of Proposition~\ref{prop:com cred}. 
		For $\com$ note that skeptical reasoning coincides with reasoning in $\grd$. 
		Finally note that by construction, $\stb$ and $\com$ coincide in the two constructions. 
		
		Regarding $\prf$ semantics membership follows by guessing and verifying some counter-example, \ie $E\in\prf(\CF)$ not containing the query argument (this can be done in $\Sigma^P_2$) and hardness is due to the AF case. 
	\end{proof}

	%	\begin{proposition}[Existence] 
		%		\label{prop:com existence}
		%		The problem $\Exists_{\sigma}$ is 
		%		trivial for $\sigma \in \{\adm,\prf\}$ and  
		%		$\NP$-complete for $\sigma \in \{\com,\grd,\stb\}$. 
		%	\end{proposition}
	%	\begin{proof}
		%		Let $\CF$ be the given BAF.
		%		
		%		Since $\emptyset\in\adm(\CF)$ and $\prf$ maximizes the finitely many $\adm(\CF)$ sets, existence is clear for these two semantics. 
		%		
		%		The case $\sigma = \stb$ follows from the corresponding reasoning problem for AFs (and the tractable verification). 
		%		
		%		We have that $\com(\CF)\neq\emptyset$ is $\NP$-hard by Lemma~\ref{le:co non-empty reduction} and membership is by tractable verification.  
		%		
		%		Finally $\com(\CF) = \emptyset$ iff $\grd(\CF) = \emptyset$. 
		%	\end{proof}

%\propConstructedPBAFWellFormed*
%\begin{proof}
%	Recall the required conditions 
%	\begin{enumerate}
%		\item $\prem(E)\subseteq \prem(E')$;  
%		\item if $\prem(E) = \emptyset$, then $E^- = \emptyset$; 
%		\item if $E$ supports some $x$, then $\prem(E)\neq \emptyset$. 
%	\end{enumerate}
%	Since $\prem(x) = \asms(x)$ for any tree-based argument $x = S\vdash p$, this follows immediately: 
%	\begin{enumerate}
%		\item $\prem(E) = \bigcup_{x\in E} \asms(x) \subseteq \bigcup_{x\in E'} \asms(x) = \prem(E')$;  
%		\item if $\prem(E) = \emptyset$, then $\asms(E) = \emptyset$ and thus $E^- = \emptyset$ by the definition of the attack relation in ABA; 
%		\item if $E$ supports $x$, then $x$ is of the form $\{a\}\vdash a$ for some assumption $a$ and $a\in \cl(\asms(E))$. 
%	\end{enumerate}
%\end{proof}

Now we turn our attention to pBAFs.

\thPBAFComplexity*

Almost all results follow since verifying exhaustiveness can be done in polynomial time. The only difference is that now, the empty set is not necessarily admissible anymore which renders skeptical reasoning more involved. 
We give the proof for this case. 

\propSkepAdmPBAF*
\begin{proof}
	(membership) 
	To show that $x\in \args$ is \emph{not} skeptically accepted, we simply need to guess a set $E\in\adm(\PF)$ s.t.\ $x\notin E$. 
	Since we can verify $E\in\adm(\PF)$ in $\P$, this procedure (proving the contrary) is in $\NP$. 
	
	(hardness) 
	For hardness, we adjust Construction~\ref{constr:co and unsat forumla} as follows: 
	We use a gadget at the bottom to ensure that each admissible extension corresponds to some assignment to the $X$-variables. To this end arguments $d_i$ are only defended if either $x_i$ or $\bar x_i$ is present. Since $\prem(d_i) = \emptyset$, these arguments occur in each admissible set (exhaustiveness). 
    
    Following the same reasonong, we use auxiliary arguments $t$ and $\bot_t$ to ensure that each admissible extension contains either $\psi$ (satisfying assignment) or $\bar{\psi}$ (no satisfying assignment). 
	We can thus check whether the formula has no satisfying assignment as follows. 
	\begin{construction}
		\label{constr:skep adm pbaf}
		Let $\Psi$ a propositional formula over atoms $X = \{x_1,\ldots,x_n\}$ in 3-CNF which we identify with a set $C$ of clauses. 
		Let $\PF_\Psi=(A,\attacker,\supporter,\prem)$ where 
		\begin{align*}
			A =& \{x_i, \bar{x}_i, \bot_i, d_i \mid 1 \leq i \leq n\} \cup \{c  \mid c \in C\} \cup \{ \psi, \bar \psi, t, \bot_t \} \\
			\attacker =& \{(x_i, \bar{x}_i), (\bar{x}_i, x_i)\mid 1 \leq i \leq n\} \cup\\   
			& \{(x_i, c) \mid x_i \in c \in C\} \cup \{(\bar{x}_i, c) \mid \neg x_i \in c \in C\} \cup \\
			& \{ (x_i,\bot_i), (\bar x_i,\bot_i),  (\bot_i,d_i) \mid 1\leq i \leq n  \}\cup \\
			& \{(c, \psi) \mid c \in C\} \cup \\
			& \{(\psi,\bar \psi), (\bot_t,t), (\psi, \bot_t), (\bar\psi, \bot_t)\} \\
			\supporter =& \emptyset \\
			\prem =& \{ (a, \{a\}) \mid a\in A\setminus \{ d_1,\ldots, d_n,t \} \} \cup \\
			& \{ (d_i, \emptyset) \mid 1\leq i \leq n \} \cup \{ (t,\emptyset) \}
		\end{align*}
		An example of this construction can be found in Figure~\ref{fig:example skep adm pbaf}. 
		
		\begin{figure}
			\begin{center}
				\begin{tikzpicture}[>=stealth,scale=0.75]
					\footnotesize
					\tikzstyle{args}=[circle,draw=black, minimum size=5mm, inner sep=0pt]
					\path 	node[args,label={above:$\{\psi\}$}](t){$\psi$}
					++(-2,0) node[args,label={left:$\{\bar \psi\}$}](top){$\bar \psi$}
                        ++(5,0) node[args,label={above:$\{\bot_t\}$}](complete1){$\bot_t$}
                        ++(2,0) node[args,label={above:$\{ \}$}](complete2){$t$}
					++(-8,-1.3) node[args,label={left:$\{c_1\}$}](c1){$c_1$}
					++(3,0) node[args,label={right:$\{c_2\}$}](c2){$c_2$}
					++(3,0) node[args,label={right:$\{c_3\}$}](c3){$c_3$};
					\path 	(-5,-3.2)  node[args,label={below:$\{x_1\}$}](x1){$x_1$}
					++(2.4,0) node[args,label={left:$\{\bar x_1\}$}](nx1){$\bar x_1$}
					++(1.3,0) node[args,label={below:$\{x_2\}$}](x2){$x_2$}
					++(2.4,0) node[args,label={left:$\{\bar x_2\}$}](nx2){$\bar x_2$}
					++(1.3,0) node[args,label={below:$\{x_3\}$}](x3){$x_3$}
					++(2.4,0) node[args,label={left:$\{\bar x_3\}$}](nx3){$\bar x_3$};
					\path 	(-5,-5)  node[args,label={below:$\{\}$}](d1){$d_1$}
					++(2.4,0) node[args,label={below:$\{\bot_1\}$}](bot1){$\bot_{1}$}
					++(1.2,0) node[args,label={below:$\{\}$}](d2){$d_2$}
					++(2.5,0) node[args,label={below:$\{\bot_2\}$}](bot2){$\bot_{2}$}
					++(1.2,0) node[args,label={below:$\{\}$}](d3){$d_3$}
					++(2.5,0) node[args,label={below:$\{\bot_3\}$}](bot3){$\bot_{3}$};
					\path [left,->, thick]	(c1) edge (t)
					(c2) edge (t)
					(c3) edge (t)
					;
	                \path [left,->, thick]
					(complete1) edge (complete2)
                        (t) edge (complete1)
                        (top) edge[bend left=50] (complete1)
					;            				
                        \path [left,->, thick]
					(x1) edge (c1)
					(x2) edge (c1)
					(nx2) edge (c2)
					(x3) edge (c2)
					(nx3) edge (c3)
					(nx1) edge (c3)
					;            
					\path [left,->, thick, bend left]
					(x1) edge (nx1)
					(nx1) edge (x1)
					(x2) edge (nx2)
					(nx2) edge (x2)
					(x3) edge (nx3)
					(nx3) edge (x3)
					;
					\path [->, thick]
					(x1) edge (bot1)
					(x2) edge (bot2)
					(x3) edge (bot3)
					(nx1) edge (bot1)
					(nx2) edge (bot2)
					(nx3) edge (bot3)
					;
					\path [->, thick]
					(t) edge (top)
					(bot1) edge (d1)
					(bot2) edge (d2)
					(bot3) edge (d3)
					;
				\end{tikzpicture}
			\end{center}
			\caption{Construction~\ref{constr:skep adm pbaf} applied to the formula $ \Psi = \{x_1,x_2\}, \{\neg x_1, x_3\}, \{\neg x_1, \neg x_3\} $}
			\label{fig:example skep adm pbaf}
		\end{figure}
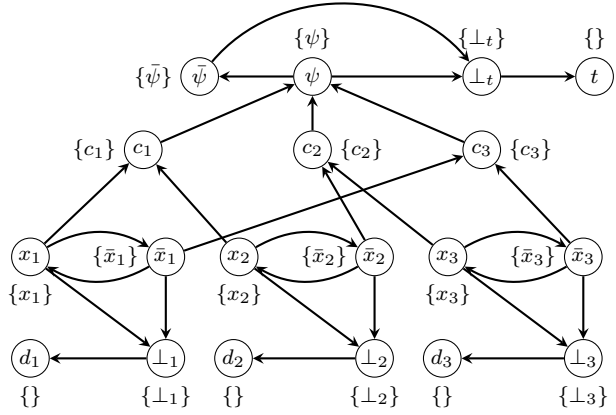
	\end{construction}
	By construction as well as the usual reasoning, 
	$\bar \psi$ is \emph{not} skeptically accepted 
	iff 
	$\Psi$ is satisfiable. 
	Thus the complement is $\NP$-complete which proves the claim. 
\end{proof}

\end{document}